\documentclass{article}

\usepackage{microtype}
\usepackage{graphicx}
\usepackage{subfig}
\usepackage{booktabs} 
\usepackage{titlecaps}
\usepackage[export]{adjustbox}

\usepackage{hyperref}

\usepackage[accepted]{icml2024}

\usepackage{amsmath}
\usepackage{amssymb}
\usepackage{mathtools}
\usepackage{amsthm}
\usepackage{bbm}
\usepackage{nicefrac}

\usepackage[capitalize,noabbrev]{cleveref}

\theoremstyle{plain}
\newtheorem{theorem}{Theorem}[section]
\newtheorem{proposition}[theorem]{Proposition}
\newtheorem{lemma}[theorem]{Lemma}
\newtheorem{corollary}[theorem]{Corollary}
\theoremstyle{definition}
\newtheorem{definition}[theorem]{Definition}

\theoremstyle{remark}
\newtheorem{remark}[theorem]{Remark}

\newcommand{\bw}{\boldsymbol{W}}
\newcommand{\bm}{\bar{M}}

\newcommand{\talpha}{\tilde{\alpha}}
\newcommand{\tm}{\tilde{m}}
\newcommand\numberthis{\addtocounter{equation}{1}\tag{\theequation}}
\makeatletter
\newcommand{\oast}{\mathbin{\mathpalette\make@circled\ast}}
\newcommand{\make@circled}[2]{%
\ooalign{$\m@th#1\smallbigcirc{#1}$\cr\hidewidth$\m@th#1#2$\hidewidth\cr}%
}
\newcommand{\smallbigcirc}[1]{%
  \vcenter{\hbox{\scalebox{0.77778}{$\m@th#1\bigcirc$}}}%
}
\makeatother

\DeclareMathOperator{\R}{\mathbb{R}}
\DeclareMathOperator{\Imm}{\mathrm{Im}}
\DeclareMathOperator{\Tr}{\mathrm{Tr}}
\DeclareMathOperator{\rank}{\mathrm{Rank}}
\DeclarePairedDelimiter\floor{\lfloor}{\rfloor}
\DeclarePairedDelimiter{\ceil}{\lceil}{\rceil}

\usepackage[textsize=tiny]{todonotes}

\icmltitlerunning{Which Frequencies do CNNs Need?}

\begin{document}

\twocolumn[
\icmltitle{Which Frequencies do CNNs Need?\\ Emergent Bottleneck Structure in Feature Learning}




\begin{icmlauthorlist}
\icmlauthor{Yuxiao Wen}{courant}
\icmlauthor{Arthur Jacot}{courant}
\end{icmlauthorlist}

\icmlaffiliation{courant}{Courant Institute of Mathematical Sciences, New York University, New York, NY 10012, USA}

\icmlcorrespondingauthor{Yuxiao Wen}{yuxiaowen@nyu.edu}
\icmlcorrespondingauthor{Arthur Jacot}{arthur.jacot@nyu.edu}

\icmlkeywords{Machine Learning, Feature Learning, CNN, Bottleneck Structure, Interpretability}

\vskip 0.3in
]



\printAffiliationsAndNotice{} 

\begin{abstract}
We describe the emergence of a Convolution Bottleneck (CBN) structure in CNNs, where the network uses its first few layers to transform the input representation into a representation that is supported only along a few frequencies and channels, before using the last few layers to map back to the outputs. We define the CBN rank, which describes the number and type of frequencies that are kept inside the bottleneck, and partially prove that the parameter norm required to represent a function $f$ scales as depth times the CBN rank $f$. We also show that the parameter norm depends at next order on the regularity of $f$. We show that any network with almost optimal parameter norm will exhibit a CBN structure in both the weights and - under the assumption that the network is stable under large learning rate - the activations, which motivates the common practice of down-sampling; and we verify that the CBN results still hold with down-sampling. Finally we use the CBN structure to interpret the functions learned by CNNs on a number of tasks.
\end{abstract}

\section{Introduction}
\label{sec:intro}
Convolutional Neural Networks (CNNs) have played a key role in the success of deep learning \cite{lecun_1998_MNIST,Krizhevsky_2012_alexnet}. It seems that the structure of CNNs is particularly well adapted to tasks on natural images. But we still lack a description of this structure, though many theories have been proposed.

The most common explanation, is that some fundamental properties of natural images are encoded in the structure of CNNs, such as translation invariance and locality. 

These intuitions have motivated special network architectures that encode additional properties such as rotation symmetries \cite{cohen_2019_equivariant_CNNs}, or the design of feature maps such as the scattering transform \cite{mallat2012group} that encode similar symmetries, upon which more traditional statistical models can then be used.

A CNN at initialization gives rise to features and kernels, either the Neural Network Gaussian Process (NNGP) kernel \cite{Neal1996,Cho2009} or the Neural Tangent Kernel (NTK) \cite{jacot2018neural}. The symmetries and invariances enforced by the locality, weight-sharing and pooling of CNNs are reflected in the kernels \cite{bietti_2019_dnns_inv_def,exact_arora2019,mei_2021_CNN_kernel_invariances,misiakiewicz_2022_CNN_kernel_invariances_pooling}, thus reducing the intrinsic dimension of the task and improving generalization \cite{mei_2021_CNN_kernel_invariances,misiakiewicz_2022_CNN_kernel_invariances_pooling}.

While the aforementioned results rely on a connection between fully-connected neural networks (FC-NNs) and kernel methods, other results have shown that the inductive bias coming from the CNN architecture is much more general, and applies to any training method that satisfies some reasonable property such as rotation equivariance \cite{li_2020_CNN_leverage_symmetries,xiao_2022_CNN_interaction_data_model_inference,Wang_2023_CNN_inductive_bias}.

But even those expertly designed kernel and features fail in general to match the performances of CNNs \cite{exact_arora2019,li2019enhanced}. A possible explanation is that feature learning allows CNNs to identify low-dimensional structures in the task during training, thus further reducing the dimensionality of the task, beyond the dimension reduction that is enforced by the CNN architecture. This is supported by the empirical observation that CNNs can learn additional symmetries during training \cite{petrini_2021_CNN_learn_diffeomorphism_invariance}. 

While there is a large literature of empirical analysis of features learned by CNNs \cite{karantzas_2022_CNN_empirical_freq_analysis} there remains very little theoretical work outside of linear CNNs \cite{dai_2021_repres_cost_DLN}.

The appearance of low-dimensional features and symmetry learning has already been observed in FC-NNs \cite{jacot_2022_BN_rank,jacot_2023_bottleneck2}. This paper extends these results to CNNs, showing a very similar bottleneck structure, though with some important differences resulted from the CNN architecture, in particular the translation invariance and pooling.


\subsection{Bottleneck Structure in CNNs}
Recent papers \cite{jacot_2022_BN_rank,jacot_2023_bottleneck2} have observed a bottleneck structure in $L_2$-regularized FC-NNs, where the representation learned in the middle layers are low-dimensional, which implies a bias towards learning symmetries.

In this paper, we extend most of the results in \cite{jacot_2022_BN_rank,jacot_2023_bottleneck2} to CNNs. An important distinction is that instead of the FC-NN bottleneck structure which favors learning any type of low-dimensional representations in the middle of the network, CNN favor representations that are supported along a finite number of frequencies, with an additional preference towards lower frequencies due to the existence of pooling:
\begin{itemize}
    \item We decompose the representation cost $R(f;\Omega,L)$ \cite{gunasekar_2018_implicit_bias2} of CNNs, which describes the implicit bias of CNNs with $L_2$-regularization, as:
    \[
    R(f;\Omega,L)=L R^{(0)}(f;\Omega) + R^{(1)}(f;\Omega) + o(1).
    \]
    \item We conjecture (and partially prove) that the first term $R^{(0)}$ equals the so-called Convolution Bottleneck rank $\mathrm{Rank}_\text{CBN}$, which is small for functions $f$ that can be decomposed as first mapping to a representation that is supported along a finite number of frequencies, with a preference for lower frequencies in the presence of pooling, and then mapping back to the outputs (that may be high dimensional and high frequency).
    \item The second term $R^{(1)}$ plays a complementary role as a measure of regularity that bounds the Jacobian of $f$.
    \item We show that under some conditions, almost all weight matrices $W_\ell$ of the network will have a few large singular values, matching the frequencies that are kept in the CBN-rank decomposition. Also, under the additional assumption that the parameters are stable under reasonable learning rate, one can show that the activations are also supported on the same few frequencies.
    \item The emergence of this bottleneck structure, where the middle representation of the network are only supported along a few low frequencies, motivates the use of down sampling, as is commonly done in practice. We show that for functions that accept such a low-frequency hidden structure, the $R^{(0)}$ term is unaffected by down-sampling in the middle of the network.
\end{itemize}
The low-dimensionality and low-frequency of the representations inside the bottleneck makes them highly interpretable. We illustrate this with a set of numerical experiments in Section~\ref{sec:numerical_exp}.

\section{Preliminaries}
\label{sec:prelim}
In this section, we first formally define the convolution operation in CNNs and related notations to express convolution in the form of matrix multiplication. Then we define the parameterization of the CNNs and their representation cost.

\subsection{Convolution in Matrix Form}
\label{sec:conv_matrix}

For any $a,b\in\mathbb{R}^n$, we define the (cyclic) convolution $a*b$ by
\[
(a*b)_i\equiv \sum_{j=1}^{n}a_jb_{i-1+j\mod n}, \quad i=1,\dots,n.
\]
The cross-channel convolution typically used in CNNs with input $x\in\R^{n\times c_1}$ and filter $w\in\R^{n\times c_2\times c_1}$ are denoted by $w\oast x\in\R^{n\times c_2}$ and defined as follows:
\[
(w\oast x)_{:,k} = \sum_{s=1}^{c_1}w_{:,k,s} * x_{:,s},\quad k=1,\dots,c_2.
\]
Note that $a*b = Ab$ with the circulant matrix
\[
A = \begin{bmatrix}
    a_1 & a_2 & \cdots & a_n\\
    a_n & a_1 & \cdots & a_{n-1}\\
    \vdots & & & \vdots\\
    a_2 & a_3 & \cdots & a_{1}
\end{bmatrix}.
\]
For the cross-channel convolution, we can also define its equivalent matrix representation by $W\in\R^{nc_2\times nc_1}$ with $W_{i,k;1,s} =w_{i,k,s}$ and $W_{i+p,k;j+p,s} = W_{i,k;j,s}$ for $i,j,p\in[n]$ and $k,s\in[c]$, where the addition is taken modulo $n$. One can verify that for $x\in\R^{n\times c_1},$
    \[
    (Wx)_{:,k} = \sum_{s=1}^{c_1}w_{:,k,s}*x_{:,s},\quad \textnormal{for }k\in[c_2].
    \]
Let $F_n\in\mathbb{C}^{n\times n}$ be the discrete Fourier transform (DFT) matrix in $n$ dimension, i.e. $(F_n)_{i,j} = \frac{1}{\sqrt{n}}\omega_n^{(i-1)(j-1)}$ where $\omega_n=e^{2\pi i/n}$. Note that $F_na$ gives the DFT coefficients of $a\in\R^n$. Also, $a*b = \sqrt{n}F^*_n\mathrm{diag}(F_na)F_nb$ where $F_n^*=F_n^{-1}$ denotes the conjugate transpose. With these matrix representations and results mentioned above, we may view convolutions as linear transformations in the Fourier domain and apply standard linear algebra results in the proofs.

\subsection{Network Parameterization}
\label{sec:paramerization}
In this paper, we consider the following parameterization of CNNs: let $x\in\Omega\subseteq\R^{n\times c_{in}}$ be the input where $\Omega$ is a compact subset, $n$ be the input size, and $c_{in}$ the number of input channels. We adopt the index convention that $A_{:,j}$ denotes the $j$-th column of a matrix $A$ and similarly for vectors and tensors. For an $L$-layer CNN, for $\ell=0,1,\dots,L-1$, the activation $\alpha_\ell(x)\in\R^{n\times c_\ell}$ at the $\ell$-th layer is defined recursively by 
\begin{align*}
\alpha_{0}(x) & =x\\
\tilde\alpha_{\ell}(x) & =\boldsymbol{1}b_\ell^T + w_\ell\oast\alpha_{\ell-1}(x)\\
\alpha_{\ell}(x)_{:,c} &= \sigma(m*\tilde\alpha_{\ell}(x)_{:,c}),\quad c=1,\dots,c_\ell
\end{align*}
where $w_\ell\in\mathbb{R}^{n\times c_\ell\times c_{\ell-1}}$ are the weight filters, $b_\ell\in\mathbb{R}^{c_\ell}$ the biases, $\boldsymbol{1}\in\mathbb{R}^n$ the all-one vector, $m\in\mathbb{R}^n$ a user-specified pooling filter applied to each channel, and nonlinearity $\sigma=\mathrm{ReLU}$. The last layer is linear:
\[
\alpha_L(x) = \tilde\alpha_L(x) = \boldsymbol{1}b_L^T + w_L\oast\alpha_{L-1}(x).
\]
As remarked in Section~\ref{sec:conv_matrix}, we write instead 
\[
\alpha_\ell(x) = \sigma\left(M(W_\ell\alpha_{\ell-1}(x)+\boldsymbol{1}b_\ell^T) \right)
\]
and focus on this matrix representation in the rest of this work. CNNs with this parameterization is naturally translationally equivariant, and discussion on its universality is deferred to Appendix~\ref{app:universality}.

\subsection{Representation Cost}
\label{sec:rep_cost}
The representation cost of a function $f$ is the minimum norm of the parameter $\theta$ for a depth-$L$ CNN $f_\theta$ to represent it over the input domain:
\[
R(f;\Omega,L) = \min_{f_\theta|\Omega = f|\Omega}\|\theta\|^2
\]
where the minimum is taken over all possible parameters $\theta=(W_1,b_1,\dots,W_L,b_L)$ with $f_\theta(x) = f(x)\ \forall x\in\Omega.$ We let $R(f;\Omega,L)=\infty$ if no such parameter exists. This representation cost describes the natural bias on the optimized CNN representation induced by introducing the $L_2$ regularization on the parameter $\theta$ for arbitrary training cost function $\mathcal{L}$:
\begin{equation}\label{eq:regularized_training_obj}
    \min_{\theta}\mathcal{L}(f_\theta) + \lambda\|\theta\|^2 = \min_{f\in\mathcal{N}_m}\mathcal{L}(f) + \lambda R(f;\Omega,L)
\end{equation}
where $\mathcal{N}_m$ denotes the set of all translationally equivariant piecewise linear (TEPL) functions that can be represented by a CNN on $\Omega$ with pooling filter $m$.
\begin{remark}
    Another natural definition for the representation cost is using the norm of the convolution filters $w_\ell$ instead of the matrix representation $W_\ell$. This only changes the parameter norm by a constant factor $\|w_\ell\|_F^2 = \frac{1}{n}\|W_\ell\|_F^2$ so that the result presented in this paper can easily be adapted to this other setting. Detailed discussion on adaptation to the filter norm is left in Appendix~\ref{app:filter_norm}.
\end{remark}

\section{Large Depth Representation Cost}
Our goal is to describe the bottleneck structure that appears in deep CNNs trained with $L_2$-regularization, e.g. Figure \ref{fig:MNIST_autoencoder}, where the weight matrices in the middle layers of the network keep only a small number of large singular values corresponding mostly to low frequencies. This bottleneck structure affects the representation cost of large depth networks.

Our intuition, which is supported by our theoretical results, is that this structure emerges because it minimizes the `cost of representing the identity': For large depths, most of the layers of the network will be dedicated to `keeping information', i.e. to represent the identity (or an orthogonal transformation) on the data. To represent the identity with a small parameter norm, it is optimal for the pre-activations to be positive, so that the ReLU equals the identity on them, and to be supported along a few low frequencies, because the weight matrix $W_\ell$ can then be chosen so that $M W_\ell$ is equals the identity along these frequencies and zero orthogonal to them. More precisely if the image of $\tilde{\alpha}_{\ell,c}$ is positive and only supported along the frequencies $I_c\subset [n]$ for each channel $c=1,\dots,c_\ell$, there we can choose $W_\ell$ such that $M W_\ell \sigma(\tilde{\alpha}_\ell(x))=\tilde{\alpha}_\ell(x)$ and 
\begin{align}\label{eq:cost_of_identity}
\|W_\ell\|_F^2 = \sum_{c=1}^{c_\ell} \sum_{i\in I_c} \tilde{m}_i^{-2}. 
\end{align}
This we call the `cost of identity' which is a sum over the cost $\tilde{m}_i^{-2}$ of representing each frequency $i$ that we keep. In the absence of pooling $M=Id$ each frequency has the same cost, but for average pooling or other types of low-pass pooling, higher frequencies have a higher cost.


\subsection{Convolutional Bottleneck Rank}
In the infinite depth limit $L\to \infty$ almost all layers will be dedicated to `representing the identity', and their parameter norm will be roughly as described in equation \ref{eq:cost_of_identity}. It is therefore optimal for the network to map in a few layers from the input representations to a representation supported along a few low frequencies, and then use the last few layers to map back to the outputs. The TEPL functions $f$ for which such a decomposition is possible are eactly those that have a small Convolutional Bottleneck (CBN) rank:


\begin{align*}
    \mathrm{Rank}_\text{CBN}(f; \Omega) \coloneqq \inf_{\substack{f=h\circ g\\
    g=g_1\oplus\cdots\oplus g_k}}\sum_{c=1}^k\sum_{i\in I_c}\tilde{m}_i^{-2}
\end{align*}
where $\oplus$ denotes channel concatenation, $f$ can be factorized into $g:\R^{n\times c_{in}}\rightarrow\R^{n\times k}$ and $h:\R^{n\times k}\rightarrow\R^{n\times c_{out}}$ for some $k\in\mathbb{N}$, $I_c\subset [n]$ denotes the supported DFT frequencies of the mapping $g_c$ of $g$ to the $c$-th channel for $c=1,\dots,k$, and $\tilde{m} = F_nm$ gives the DFT coefficients of the pooling $m$.

TEPL functions $f$ with a small CBN rank can be represented by a deep CNN with a small parameter norm:
\begin{theorem}\label{thm:CBN_upper_bound}
    For any translationally equivariant function $f$ with finite CBN rank, there is a constant $c$ that depends only on the target function $f$ s.t.
    \[
    R(f;\Omega,L) \leq L \mathrm{Rank}_{\textnormal{CBN}}(f; \Omega) + c.
    \]
\end{theorem}
\begin{proof}
    We sketch the proof idea of this theorem here. Suppose $f=h\circ g$ attains the bottleneck rank. By Lemma~\ref{lemma:bounded_depth}, there is a CNN with depth $L_g= \floor{\log(nc_{in}+1)}+2$ ($L_h= \floor{\log(nk+1)}+2$ resp.) and parameter $\theta_g$ ($\theta_h$ resp.) that represents $g$ ($h$ resp.). For $L\geq L_g+L_h$, we can construct the following CNN that represents $f$:
    Let the first $L_g$ layers represent $g$ and the last $L_h$ layers represent $h$. Let the middle $L-L_g-L_h$ layers be identity layers on $g(\Omega)$. The overall parameter norm of this CNN is
    \[
    \|\theta\|^2 = \|\theta_g\|^2 + \|\theta_h\|^2 + (L-L_g-L_h)\sum_{c=1}^k\sum_{t\in I_c}\tm_t^{-2}.
    \]
\end{proof}

If we define the limiting representation cost $R^{(0)}(f;\Omega)$ as the limit $\lim_{L\to\infty} \frac{1}{L} R(f;\Omega,L)$, then this implies the upper bound $R^{(0)}(f;\Omega)\leq \mathrm{Rank}_\text{CBN}(f;\Omega)$, but we conjecture that the two are actually equal. This conjecture is inspired by the fact that in numerical experiments, there is a similar bottleneck structure as the one in the proof of Theorem \ref{thm:CBN_upper_bound}, suggesting that such a structure might be indeed optimal (up $o(L)$ terms).

We give the following theoretical support for our conjecture. First the $R^{(0)}$ shares a number of properties typical of a notion of rank with $\mathrm{Rank}_\text{CBN}$, such as 
\begin{align*}
    R^0(f_2\circ f_1; \Omega) &\leq \min\{R^0(f_2; f_1(\Omega)), R^0(f_1; \Omega)\};\\
    R^0(f_2+ f_1;\Omega) &\leq R^0(f_2; \Omega) +R^0(f_1; \Omega).
\end{align*}
These properties as well as others are proven in Appendix~\ref{app:properties}.

Second, and more importantly, we give a lower bound for $R^{(0)}(f;\Omega)$ in terms of the Jacobian $Jf(x)$ at a point $x$, which matches the upperbound for a large family of TEPL functions $f$.

This lower bound will be expressed in terms of the following pooling-dependent rank: for any translation equivariant matrix $A\in\R^{n\min\{c_{in},c_{out}\}\times n\min\{c_{in},c_{out}\}}$, define
\begin{align*}
    \mathrm{Rank}_{m}(A) = \sum_{c=1}^{\min\{c_{in},c_{out}\}}\sum_{t=1}^{n}\tm_t^{-2}\mathbbm{1}[s_{c,t}(A)\neq 0]
\end{align*}
where $s_{c,t}(A)$ denotes the $c$-th singular value of $A$ along the $t$-frequency for $c=1,\dots,\min\{c_{in},c_{out}\}$ and $t=1,\dots,n$. Note that in the absence of pooling (i.e. $m=id$), it reduces to the matrix rank $\mathrm{Rank}_{id}(A) = \mathrm{Rank}(A)$.

\begin{theorem}\label{thm:m_lower_bound}
    For any translationally equivariant function $f$, let $Jf(x)$ be the Jacobian of $f$ at $x$. The following pooling-dependent lower bounds hold:
    \begin{enumerate}
        \item $\frac{1}{\tm_{\max}^2}\max_{x\in\Omega}\mathrm{Rank}(Jf(x)) \leq R^{(0)}(f;\Omega)$
        
        In particular, when there is no pooling, $\max_{x\in\Omega}\mathrm{Rank}(Jf(x)) \leq R^{(0)}(f;\Omega)$.
        
        \item $\max_{x\in\Omega_-}\mathrm{Rank}_m(Jf(x)) \leq R^{(0)}(f;\Omega)$ 
        
        where the max is taken over the subset $\Omega_-\coloneqq\{x\in\Omega\,|\, x_{p,i}=x_{q,i}\ \forall i=1,\dots,c_{in}, \forall p,q=1,\dots,n \}$, i.e. all $x$ that are constant along each channel.
    \end{enumerate}
\end{theorem}

If there is a point $x\in\Omega_-$ (or $x\in\Omega$ when there is no pooling) that matches the lower bound in Theorem~\ref{thm:m_lower_bound} and the upper bound Theorem~\ref{thm:CBN_upper_bound}, we prove the conjecture that $R^{(0)}=\mathrm{Rank}_{\text{CBN}}$. For example, if the target function $f$ is a linear one-layer CNN and $\exists x\in\Omega_-$ is an interior point in $\Omega$, by matching the upper and the lower bounds, we have
\[
R^{(0)}(f;\Omega) = \mathrm{Rank}_{\text{CBN}}(f;\Omega) = \sum_{c=1}^{\min\{c_{in}, c_{out}\}}\sum_{t\in I_c}\tilde{m}_t^{-2}
\]
where $I_c$ is the DFT frequencies supported by the weight filter $W$ at the $c$-th output channel (see proof in Appendix~\ref{app:properties}).

\subsection{Finite Depth Correction}
There are many functions with the same CBN rank, some more complex than others, depending on the complexity of the functions $g$ and $h$. The $R^{(0)}$-term fails to capture the complexity of $g$ and $h$ as can be seen in the sketch of proof of Theorem~\ref{thm:CBN_upper_bound}, where the corresponding parameter norms $\|\theta_g\|^2 $ and $\|\theta_h\|^2$ have negligible contribution to the parameter norm in contrast to the middle identity layers and do not affect $R^{(0)}$. To capture these subdominant terms, we consider the following correction term:

\begin{definition}\label{def:R1}
    Define the finite depth correction term by
    \[
    R^{(1)}(f;\Omega) \coloneqq \lim_{L\rightarrow\infty}R(f;\Omega,L) - LR^0(f;\Omega).
    \]
\end{definition}

This correction term $R^{(1)}$ serves as a ``regularity control" on the learned CNNs:

\begin{proposition}\label{prop:R1_properties}
\begin{enumerate}
    \item For any $x\in \Omega_-$, $R^{(1)}(f;\Omega)\geq 2\sum_{s_{c,t}\neq 0}\tilde{m}_t^{-2}\log (s_{c,t}\tm_t)$ with $s_{c,t}$ being the $(c,t)$-th singular values of $Jf(x)$ for $c=1,\dots,\min\{c_{in},c_{out}\}$ and $t=1,\dots,n$.
    \item For all $x\in \Omega$, if there is no pooling, then $R^{(1)}(f;\Omega)\geq 2\log |Jf(x)|_+$.
    \item If $R^{(0)}(f\circ g;\Omega) = R^{(0)}(f;g(\Omega)) = R^{(0)}(g;\Omega)$, then $R^{(1)}(f\circ g;\Omega) \leq R^{(1)}(f;g(\Omega))+ R^{(1)}(g;\Omega)$.
    \item If $R^{(0)}(f+g;\Omega) = R^{(0)}(f;\Omega) + R^{(0)}(g;\Omega)$, then $R^{(1)}(f+ g;\Omega) \leq R^{(1)}(f;\Omega)+ R^{(1)}(g;\Omega)$.
\end{enumerate}
\end{proposition}

As shown by the first and second point of Proposition~\ref{prop:R1_properties}, the finite depth correction $R^{(1)}(f;\Omega)$ controls the regularity of the learned function $f$ by upper bounding the (weighted) sum of the log singular values of the Jacobian. The third statement in Proposition~\ref{prop:R1_properties} indicates that among functions with the same $R^{(0)}$ cost, their ``regularity control" satisfies subadditivity. 

We can rewrite the $L_2$-regulartized training objective in Equation~\ref{eq:regularized_training_obj} approximately in terms of $R^{(0)}$ and $R^{(1)}$:
\begin{equation}\label{eq:regularized_obj_R0R1}
    \min_{f\in\mathcal{N}_m}\mathcal{L}(f) + \lambda LR^{(0)}(f;\Omega) + \lambda R^{(1)}(f;\Omega)
\end{equation}
where the depth $L$ now plays a role of balancing the rank estimation and the regularity control. If our conjecture $R^{(0)}(f;\Omega) = \mathrm{Rank}_\text{CBN}(f;\Omega)$ holds, then we may classify the functions $f\in\mathcal{N}_m$ into subsets according to their BN-rank $R^{(0)}(f;\Omega)$, i.e. for each possible combination $I_k\in\mathcal{P}([\min\{c_{in},c_{out}\}]\times[n])$ we can define 
\[
\mathcal{N}_{m,k}\coloneqq \left\{f\in\mathcal{N}_m\,:\, R^{(0)}(f;\Omega) = \sum_{(c,t)\in I_k}\tm_t^{-2}\right\}.
\]
For fixed depth $L$ and within each $\mathcal{N}_k$, the objective \ref{eq:regularized_obj_R0R1} minimizes the loss and the $R^{(1)}$ term that controls the regularity via $\min_{f\in\mathcal{N}_{m,k}}\mathcal{L}(f) + \lambda R^{(1)}(f;\Omega)$ and hence the objective itself becomes
\[
\min_{k\in[K]}\left\{\lambda L\sum_{(c,t)\in I_k}\tm_t^{-2} + \min_{f\in\mathcal{N}_k}\mathcal{L}(f) + \lambda R^{(1)}(f;\Omega)\right\}.
\]
This reformulated objective suggests that for each possible bottleneck rank, indexed by $k\in[K]$, there is a regular minimizer $f_k\in\mathcal{N}_{m,k}$, and the depth $L$ only decides which $f_k$ is the global minimizer by trading off the bottleneck rank term and the inner minimization term (which controls the regularity of $f_k$). This reformulated objective suggests that as $L\rightarrow\infty$, regularized training is biased toward low-rank CNNs whose inner representations concentrate to frequencies where most information is kept (with large $\tm_t$).

\section{Bottleneck Structure in Weights and Pre-Activations}
Although we cannot prove the conjecture in its entirety, we are indeed able to show a bottleneck structure in the weights and the pre-activations of CNNs with sufficiently small parameter norms. 

\noindent \titlecap{\scshape bottleneck structure in weights}\normalfont

In the proof of Theorem~\ref{thm:CBN_upper_bound}, we construct a CNN where the weights in most layers support only a few frequencies. This bottleneck structure in the weights is also observed in the numerical experiments in Section~\ref{sec:numerical_exp}. We show that when the parameter norm is sufficiently small, this bottleneck structure is common in the weights:
\begin{theorem}\label{thm:bottleneck_weights}
    Suppose $\exists k>0$ such that the parameter norm $\|\theta\|^2\leq kL+c$ is small enough for $k= \max_{x\in\Omega_-}\mathrm{Rank}_{m}(Jf_\theta(x))$. Let $x_0\in \mathrm{argmax}_{x\in \Omega_-}\mathrm{Rank}_mJf_\theta(x)$. Then there are $V_\ell^T\in\R^{\kappa\times nc_{\ell-1}}$ and $U_\ell\in\R^{nc_\ell\times\kappa}$ being submatrices of the DFT block matrices $F_{\ell-1}\in\R^{nc_{\ell-1}\times nc_{\ell-1}}$ and $F_\ell^T\in\R^{nc_{\ell}\times nc_{\ell}}$ respectively, where $\kappa = \mathrm{Rank}Jf_\theta(x_0)$, such that 
    \begin{equation*}
        \resizebox{.95\hsize}{!}{$\displaystyle \sum_{\ell=1}^L\|W_\ell-U_{\ell} S_\ell V_\ell^T\|^2_F + \|b_\ell\|^2_F \leq c - 2\sum_{s_{t,c}\neq 0}\tm_t^{-2}\log (s_{t,c}\tm_t)$}
    \end{equation*}
    and thus for any $p\in (0,1)$, there are at least $(1-p)L$ layers $\ell$ with
    \begin{equation*}
        \resizebox{.95\hsize}{!}{$\|W_\ell-U_{\ell} S_\ell V_\ell^T\|^2_F+ \|b_\ell\|^2_F \leq \frac{c - 2\sum_{s_{t,c}\neq 0}\tm_t^{-2}\log (s_{t,c}\tm_t)}{pL}$}
    \end{equation*}
    where $s_{t,c}$ is the $(t,c)$-th singular value of $Jf_\theta(x_0)$ and $S_\ell\in\R^{\kappa\times\kappa}$ is a diagonal matrix with entries $\in\{\tm_t^{-1}\}_{t=1}^{n}$.
\end{theorem}
The assumptions on the norm $\|\theta\|^2$ and the Jacobian $Jf_\theta(x)$ in Theorem \ref{thm:bottleneck_weights} are there to guarantee that we are in the setting where the upper and lower bounds on $R^{(0)}$ of Theorems \ref{thm:CBN_upper_bound} and \ref{thm:m_lower_bound} match up to a constant, and that the network represents $f_\theta$ with almost minimal parameter norm. The proof leverages the small gap between the lower and upper bound to prove the bottleneck structure (using the fact that an inequality can only be satisfied with almost equality under certain conditions). We hope that proving the conjecture that $R^{(0)} = \mathrm{Rank}_\text{CBN}$ would make it possible to alleviate these assumptions, as there would be a lower bound that matches the upper bound for all functions instead of only some functions. 

While at the minimal norm parameters $\theta$, we know the residual term $c_1$ approaches and is upper bounded by $R^{(1)}(f_\theta;\Omega)$, this result also generalizes to all approximately minimal norm parameters where $c_1$ is still close to $R^{(1)}(f_\theta;\Omega)$. The fact that it generalizes implies that this bottleneck structure in the weights manifests in an "almost optimal" region around the optimal parameters into which the regularized objective eventually falls.

\noindent \titlecap{\scshape bottleneck structure in pre-activations without pooling}\normalfont

The fact that almost all weight matrices $W_\ell$ are supported along only a finite number of frequencies suggests that the corresponding pre-activations $\tilde{\alpha}_\ell(X) = W_\ell \alpha_{\ell-1}(X) + b_\ell$ for any training set $X$ should also be supported along the same frequencies (and possibly also along an additional constant frequency because of the bias term). 

This is trivial if the activations remain bounded as the depth $L$ grows, but \cite{jacot_2023_bottleneck2} has shown a counterexample: a simple function whose optimal intermediate representations explode in the infinite depth limit. This couterexample can easily be translated to the CNN setup (by applying the same function in parallel to all pixels of a constant input). This implies that to guarantee bounded representations in general, we need another source of bias, in addition to the small parameter norm bias. Following \cite{jacot_2023_bottleneck2}, we turn to the implicit bias of large learning rates in GD.

We know that GD with a learning rate of $\eta$ can only converge to a minima $ \hat{\theta}$ where the top eigenvalue of the Hessian $\lambda_1(\mathcal{H}\mathcal{L}_\lambda (\hat{\theta}))$ is upper bounded by $\nicefrac{2}{\eta}$. Other results suggest that SGD is biased towards minima where the trace of the Hessian is small \cite{damian2021_label_noise_trace_Hessian,li2021_SGD_zero_loss}. The top eigenvalue and trace both are measures of the narrowness of the minimum.

For the MSE loss, the Hessian at a local minimum $\hat{\theta}$ that fits the data (in the sense that $\mathcal{L}_\lambda(\hat{\theta})=O(\lambda)$) takes the form
\[
\mathcal{HL}_\lambda(\hat{\theta}) = \frac{2}{N}\sum_{i=1}^N J_\theta f_\theta(x_i)^T J_\theta f_\theta(x_i) + O(\lambda).
\]
The trace of Hessian is then approximately equal to $\frac{2}{N}\sum \|J_\theta f_\theta(x_i)\|_F^2$ and the largest eigenvalue is lower bounded by $\frac{2}{d_{out} N^2 n}\sum \|J_\theta f_\theta(x_i)\|_F^2 - O(\lambda)$ since the first term has rank $N n d_{out}$.

The term $\|J_\theta f_\theta(x)\|_F^2$ (which also equals $\mathrm{Tr}[\Theta(x,x)]$ for $\Theta$ the NTK \cite{jacot2018neural}) typically scales linearly in depth since it equals the sum over the $L$ terms $\|J_{(W_\ell,b_\ell)}f_\theta(x)\|_F^2$, so a choice of learning rate $\eta=O(L^{-1})$ is natural. This forces convergence to a minimum with $\|J_\theta f_\theta(x)\|_F^2 \leq cL$ which in turns implies that almost all activations are bounded:


\begin{theorem}\label{thm:bounded_activations}
    Given a depth $L$ network without pooling, balanced parameters $\theta$ with $\|\theta\|^2\leq Lk + c_1$ for $k= \max_{z\in\Omega_-}\rank_m(Jf_\theta(z))$, and a point $x_0$ such that $\rank Jf_\theta(x_0)=\max_{z\in\Omega_-}\rank_m(Jf_\theta(z))$, then $\|J_\theta f_\theta(x_0)\|_F^2 \leq cL$ implies that,
    \[
    \sum_{\ell=1}^L\|\alpha_{\ell-1}(x_0)\|_2^2 \leq \frac{ce^{\frac{c_1}{k}}}{k|Jf_\theta(x_0)|_+^{\nicefrac{2}{k}}}L.
    \]
    Hence for each $p\in(0,1)$, there are at least $(1-p)L$ layers $\ell$ with
    \[
    \|\alpha_{\ell-1}(x_0)\|_2^2 \leq \frac{1}{p}\frac{ce^{\frac{c_1}{k}}}{k|Jf_\theta(x_0)|_+^{\nicefrac{2}{k}}}\,.
    \]
\end{theorem}


Theorem~\ref{thm:bounded_activations} gives the conditions under which the activations remain bounded, and thereby the pre-activations $\talpha_\ell(X)$ are supported along the same frequencies as $W_\ell$ (in which Theorem~\ref{thm:bottleneck_weights} proves a bottleneck structure) and possibly also the constant frequency. Note that the results we present in this section do not require the CNNs to be well-trained to (approximately) global minimums.

\section{CNNs with up and down-sampling}
Given the bottleneck structure in the near-optimal parameters, it is natural to consider implementing down-sampling and up-sampling in CNNs which explicitly enforce a bottleneck structure and save computational cost, as commonly used in practice. In this section, we study CNNs with down-sampling and up-sampling layers.

\begin{figure*}
\vspace{-0.4cm}
\begin{minipage}{.33\textwidth}
    \includegraphics[scale=0.36,valign=t]{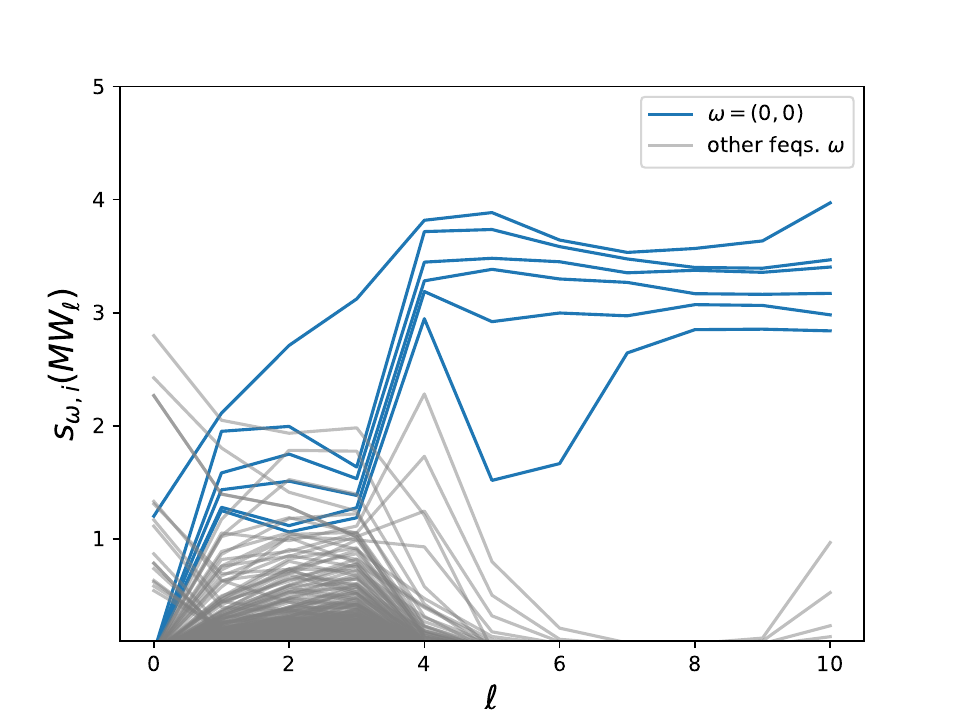}
    
\caption{\label{fig:MNIST_classifier} We train a CNN ($L=11,c_\ell=60,\lambda=0.005,\beta=0.5$) on MNIST. The inputs are $28\times28$, and scaled down by 2 on the 2nd and 4th layers, with global average pooling and a fully connected layer at the end.  We see that for classification, six constant frequencies are kept.}
\end{minipage}\;\;\;\;%
\begin{minipage}{.66\textwidth}
\centering
\subfloat[Sing. vals. of $M W_\ell$]{\includegraphics[scale=0.33]{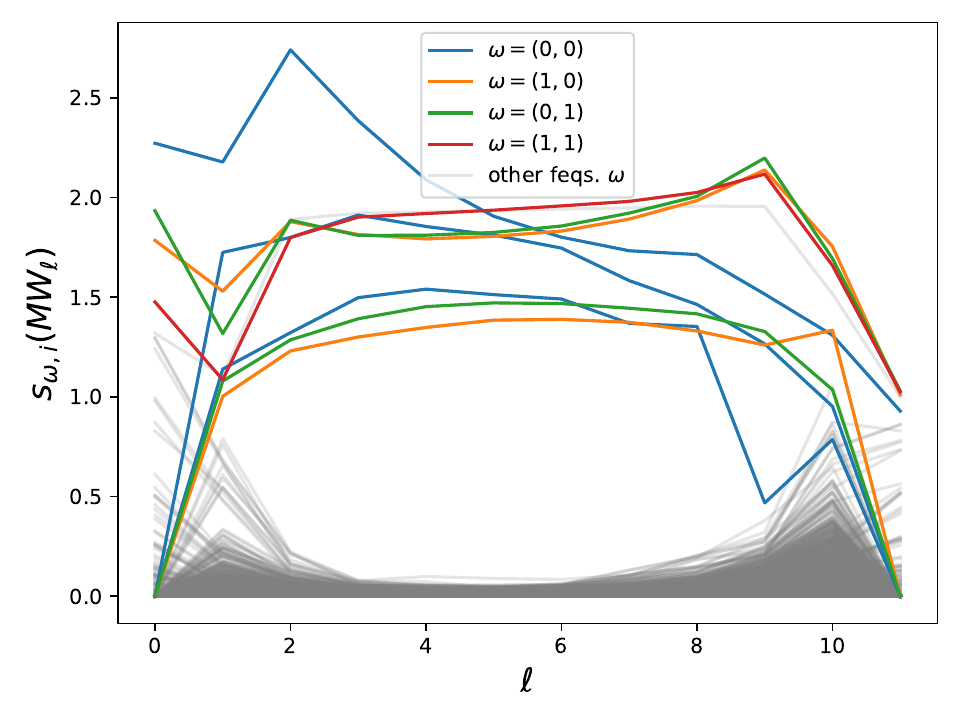}
}
\subfloat[Latent space interpretation]{\includegraphics[scale=0.36]{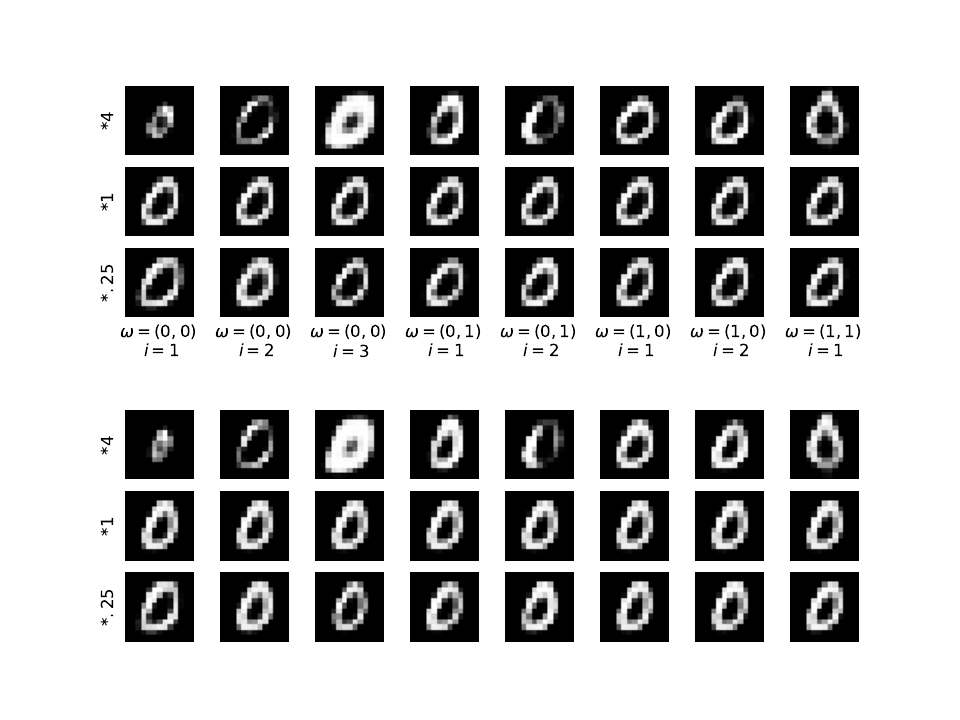}}

\hspace{-0.5cm}

\caption{\label{fig:MNIST_autoencoder} We train an autoencoder ($L=12,c_\ell=50,\lambda=0.04,\beta=1.0$) on the $0$-digits of MNIST downscaled to the size $13\times 13$. (a) The singular values of $MW_\ell$ for every layer $\ell$, colored by their frequency $\omega$. (b) Along each of the singular values in the $5$-th layer, we plot the effect of multiplying the hidden representation along the sing. vector by 2 or 0.5 (for non-constant frequencies we also consider multiplication by complex $i$ and $-i$). We see how each singular value correspond to a (nonlinear) direction of variation of the zeros. For non-constant frequencies the argument encodes the $x$ and $y$ position of the digit.}

\end{minipage}
\end{figure*}

We only consider CNNs with one down-sampling layer and one up-sampling layer, but the results can be extended to having multiple such layers. To be specific, consider the set of CNNs with the parametrization in Section~\ref{sec:paramerization} but with one down-sampling layer and one up-sampling layer inserted: Let $\mathcal{N}_{n;m}$ be all possible CNNs with any depth, input dimension $n$, and pooling $m$. For each stride $s\in\mathbb{N}$, define the set of stride-CNNs with inner pooling $m'$:

\begin{align}
    \begin{split}\label{eq:stride-CNNs}
    \mathcal{N}^{(s)}_{n;m,m'} = &\{f_2\circ \mathrm{Up}_s\circ \hat{f}\circ \mathrm{Down}_s\circ f_1:\\
    &f_1, f_2\in\mathcal{N}_{n;m}, \hat{f}\in\mathcal{N}_{\floor{\nicefrac{n}{s}};m'}\}
    \end{split}
\end{align}

Because of the down-sampling layer, our networks are no longer translationally equivariant; instead, they only represent $s$-translationally equivariant functions (i.e. invariant under translations by multiples of $s$). The formal mathematical definitions for down-sampling and up-sampling in (\ref{eq:stride-CNNs}) are as follows:

\begin{definition}
    Define the \textbf{down-sampling} operator $\mathrm{Down}_s: \R^{n\times c}\rightarrow \R^{\floor{n/s}\times c}$ by mapping $(\mathrm{Down}_s(x))_{i,k} = x_{si,k}$, i.e. subsampling every $s$ pixel along each channel. 

    Define the \textbf{up-sampling} operator $\mathrm{Up}_s: \R^{n'\times c}\rightarrow \R^{n's\times c}$ by 
    \[
    \mathrm{Up}_s(x) = F_{n's}^*[sI_{n'},  0]^TF_{n'}x
    \]
    i.e. mapping the $n'$ Fourier coefficients of input $x$ to the first $n'$ Fourier coefficients of $\mathrm{Up}_s(x)$ and zeros otherwise. $F_N$ denotes the DFT matrix of dimension $N\times N$.
\end{definition}

\begin{remark}
    By the Nyquist-Shannon sampling theorem \cite{shannon1949}, we have that for $y=\mathrm{Down}_s(x)$, the $i$-th DFT coefficient is $\tilde{y}_i = \frac{1}{s}\sum_{j=0}^{s-1}\tilde{x}_{i+\nicefrac{jn}{s} \mod n}$. Hence exact reconstruction of $x$ is possible when $x$ is low-frequency i.e. $\tilde{x}_i=0$ for $i\geq \frac{n}{s}$ (in which case the set of coefficients $\{\tilde{x}_{i+\nicefrac{jn}{s}}\}_{j=0}^{s-1}$ has cardinality $\leq 1$ and gives a one-to-one mapping between the coefficients of $x$ and $y$).
\end{remark}

\begin{remark}
    The set of all $s$-stride-CNNs $\mathcal{N}^{(s)}_{n;m,m'}$ is the set of all functions $f=h\circ g$ with bottleneck support only on the first $\frac{n}{s}$ DFT frequencies, cf. Proposition~\ref{prop:Ns_decomposition}.
\end{remark}

In other words, if a full-size CNN can be decomposed into two TEPL functions with only low frequencies, then it can be represented by a CNN with down and up-sampling. We thereby have the natural extension of the CBN rank for the stride-CNNs:
\[
\mathrm{Rank}_\text{CBN}^{(s)}(f;\Omega)\equiv \min_{\substack{f=h^{(s)}\circ g^{(s)}\\g^{(s)}=g^{(s)}_1\oplus\cdots\oplus g^{(s)}_k}} \sum_{c=1}^k\sum_{i\in I_c}\tm_i'^{-2}
\]

where in the decomposition $f=h^{(s)}\circ g^{(s)}$, $g_c^{(s)}$ on each channel $c$ only supports low frequencies $I_c\subseteq [\frac{n}{s}]$. Note that any $f\in\mathcal{N}^{(s)}_{n;m,m'}$ has finite stride-bottleneck rank $\mathrm{Rank}_\text{CBN}^{(s)}(f;\Omega)<\infty$. If the inner pooling $m'$ (of size $\frac{n}{s}$) is the same as $m$ (of size $n$) truncated to the first $\frac{n}{s}$ frequencies, then it is straightforward that $\mathrm{Rank}_\text{CBN}(f;\Omega) \leq \mathrm{Rank}_\text{CBN}^{(s)}(f;\Omega)$.

\begin{remark}
    The reason for having the first $\frac{n}{s}$ frequencies here is due to the choice we made in the up-sampling operator. One can slightly generalize to exact reconstruction of $x$ consisting of another set of $\frac{n}{s}$ frequencies by having a different mapping between the low-dimensional and the high-dimensional Fourier coefficients, as long as the input satisfies $|\{\tilde{x}_{i+\nicefrac{jn}{s}}\}_{j=0}^{s-1}|\leq 1$ for each $0\leq i <\frac{n}{s}$.
\end{remark}

Furthermore, we may recover the upper bound theorem as in Theorem~\ref{thm:CBN_upper_bound}.

\begin{theorem}
    Let $R^{(0)}_s(f;\Omega)$ denote the rescaled representation cost under the architecture with stride $s$. Then for any $f\in\mathcal{N}^{(s)}_{n;m,m'},$ 
    \[
    R^{(0)}_s(f;\Omega) \leq \mathrm{Rank}_\text{CBN}^{(s)}(f;\Omega).
    \]
\end{theorem}
\begin{proof}
    The proof idea follows from that of Theorem~\ref{thm:CBN_upper_bound}. Suppose $f = h^{(s)}\circ g^{(s)}$ realizes $\mathrm{Rank}_\text{CBN}^{(s)}(f;\Omega)$. Observe that $f = h^{(s)}\circ \mathrm{Up}_s \circ \hat{f}\circ\mathrm{Down}_s\circ g^{(s)}$ where $\hat{f}$ consists of $\hat{L}$ identity layers and hence $\mathrm{Up}_s \circ \hat{f}\circ\mathrm{Down}_s=id|_{\Imm g^{(s)}}$. The bound follows by taking $\hat{L}$ to infinity.
\end{proof}

\begin{remark}
    One may also generalize the properties of $R^{(0)}$ to $R^{(0)}_s$ following the same proof ideas. 
\end{remark}

If the target function possesses a good low-frequency bottleneck structure in the sense that $\mathrm{Rank}_\text{CBN}\approx\mathrm{Rank}_\text{CBN}^{(s)}$, under the conjecture $R^{(0)}=\mathrm{Rank}_\text{CBN}$ and $R^{(0)}_s = \mathrm{Rank}_\text{CBN}^{(s)}$, we can see that $R^{(0)}\approx R^{(0)}_s$ (meaning their optimal representation costs are close). Hence one is justified to learn the target with CNNs with enforced down-sampling and up-sampling layers for reduced computation cost and lower-dimensional latent representations in the Euclidean space.

\noindent \titlecap{\scshape low frequency representation}\normalfont

Although we show exact recovery is possible with appropriate stride $s$, there remains the question of how to choose the stride for down-sampling in our CNNs a priori. To partially answer this question, under some realistic assumptions on the input domain $\Omega$, we can show that the target function $f:\Omega\rightarrow\R^{n\times c_{out}}$ has a low-frequency decomposition $f=h\circ g^{(2)}$ with stride $s=\frac{n}{2}$ (i.e. inner representations only have input size $2$ and hence only support 2 frequencies). Yet we remark that having a $2$-frequency decomposition does not imply that the optimal stride is of size $2$, because low-frequency decomposition may require too many channels for exact recovery, whereas retaining a few more frequencies may be more informative and efficient.

\begin{definition}
The input domain $\Omega$ is \textit{translationally unique} if $\forall x,y\in\Omega, x= T_py\implies x=y$ and $p=0$, where $T_p$ denotes the translation by $p$ along each channel for $p=0,\dots,n-1$.
\end{definition}

In particular, for this kind of domain, $\Omega_-=\emptyset$. Though it is difficult to check or guarantee that all natural images are translationally unique, it seems to hold for the vast majority of images.

\begin{theorem}\label{thm:translationally_unique_domain}
    Suppose $\Omega$ is translationally unique. Then for any piecewise linear target function $f:\Omega\rightarrow\R^{n\times c_{out}}$, $f=h\circ g^{low}$ where $h$ and $g^{low}$ are TEPL functions and $g^{low}:\Omega\rightarrow\R^{n\times nc_{in}+1}$ only supports the constant DFT frequency at first $nc_{in}$ channels and the second DFT frequency at the $nc_{in}+1$-th channel.
\end{theorem}

\begin{remark}\label{rmk:trans_unique_domain}
Theorem~\ref{thm:translationally_unique_domain} implies that the identity map on translationally unique domains can be represented using $nc_{in}$ constant frequencies and one 1-periodic frequency. In particular, it gives an upper bound on the bottleneck rank of any TEPL function $f$ on such domain $\Omega$, including $id$, that
    \[
    \mathrm{Rank}_\text{CBN}(f;\Omega) \leq \tm_2^{-2} + nc_{in}\tm_1^{-2}.
    \]
\end{remark}

\section{Numerical Experiments}
\label{sec:numerical_exp}
For our numerical experiments, we train networks on 4 different tasks, with different depths and ridge parameters. We use filters with full size and cyclic boundaries. The pooling operator is $M_\beta=(1-\beta)I + \beta A_3$, where $A_3$ is the $3\times3$ average filter. We use a few different values of $\beta$. For the MNIST classification task, we also implement downsampling in the 2nd and 4th layers. The experiments are done for 2D convolution instead of 1D convolution as in the theoretical analysis, but everything translates directly, with the difference that frequencies are indexed by pairs $\omega$.

The emergent bottleneck structure that appears in all the tasks we consider makes these networks highly interpretable. We plot the singular values $s_{\omega,i}(M W_\ell)$ accross the layers $\ell=1,\dots,L$. We emphasize the singular values that are kept in the bottleneck by coloring them according to their frequency.

\textbf{MNIST classification:} For MNIST classification the CNN features a global pooling layer at the end, followed by a final fully-connected layer. This explains why only constant frequencies are kept in the bottleneck, since any non-constant frequencies in the outputs are killed by the global pooling. Only 6 dimensions are kept, which is sufficient to embed all 10 classes in a linearly separable manner.

Also note that this experiments illustrates a `half-bottleneck', where the representations go from high-dim/high-freq inputs to a low-dim/low-freq bottleneck and remain there until the outputs. This is in contrast to the full bottlenecks that we observe in our other experiments where the representations go back to high-dim/high-freq in the last layers. Note that this half-bottleneck structure (which is common in classification tasks since the outputs of the network are low-dim/low-freq) could explain some aspects of the neural collapse phenomenon \cite{papyan_2020_neural_collapse} as well as other numerical observations \cite{kornblith_2019_similarity}.

\textbf{MNIST digit 0 autoencoder:} When training an autoencoder the networks keeps3 constant freq. along with 4 degree 1 freq. and 1 degree two freq. Since the signal inside the bottleneck is almost only supported along low frequencies, the middle layers could have been downsampled before upsampling again (as is usually done with autoencoders), but the $L_2$-regularization alone recreates the same effect. We believe allowing the network to choose the frequencies it wants to keep and the number of channels is better than forcing it. Of course there are computational advantages to downsampling in the middle of the network.

To understand what each of the kept frequencies capture, we plot the effect of multiplying by $4$ or $0.25$ the signal along each singular value of $W_5$ and plotting the resulting modified output. The effect along some singular values can be interpreted as capturing e.g. size, boldness, narrowness, angle and more.

\textbf{Autoencoder on synthetic data:} We train an autoencoder on data obtained as the pixelwise multiplication of a low-freq shape with a high-freq repreating pattern (a single freq.-$(5,5)$ Fourier function with random phase). We see that the network disentangles the shape from the pattern in the bottleneck, the shape is encoded in the $\|\omega \|_1 \leq 2$-freqs and the pattern in the single $(5,5)$-freq. This is only possible with non-linear transformations at the beginning and end of the network.

\textbf{Learning Newtonian Mechanics:} We train a network to predict the trajectory of a ball: the inputs to the network are four frames of a ball under gravity (with different frames encoded in different channels) with a random initial position and velocity, from which the network has to predict the next 4 frames. The network keeps two pairs of degree one frequencies (and one constant frequency, which seems to only be there to ensure that the signal remains in $R_+$ inside the bottleneck; one can check that no information is kept in this constant frequency). The phases of the largest pair of degree one frequencies $\theta_1$ and $\vartheta_1$ encode the $(x,y)$-position of the ball two frames before the end, and the difference in phases between the largest pair and the smaller pair encodes the $(x,y)$-velocity at the same frame. Thus the network recognizes that the evolution of the ball is uniquely determined by its position and velocity.


%


\begin{figure}[h]\centering
\subfloat[Singular values of $M W_\ell$]{\includegraphics[scale=0.35,valign=t]{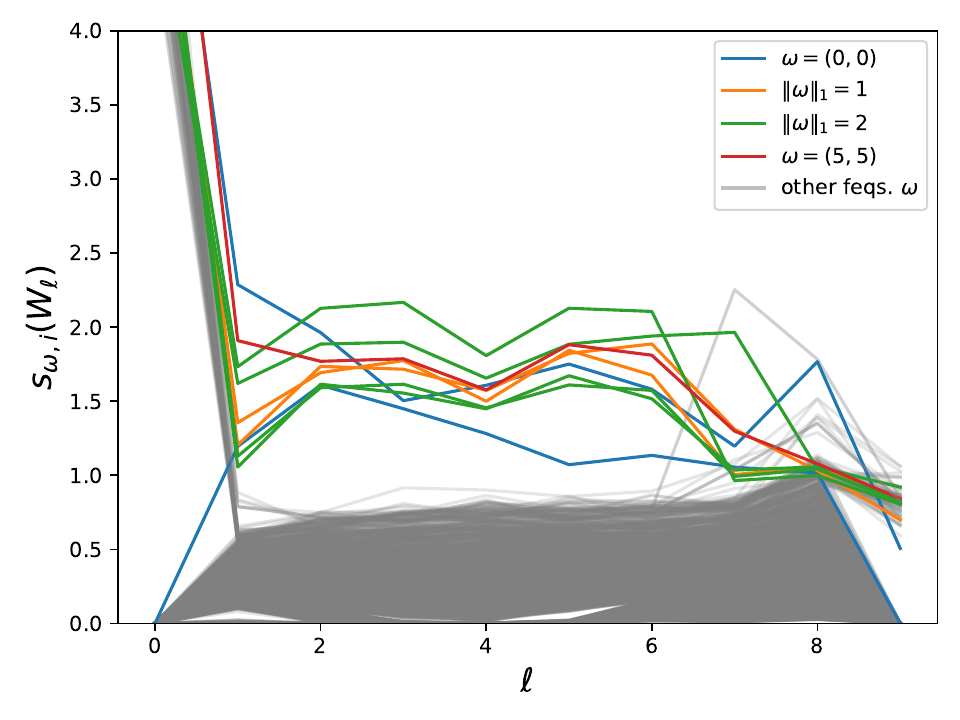}
}
\subfloat[Training data]{\includegraphics[scale=0.38,valign=t]{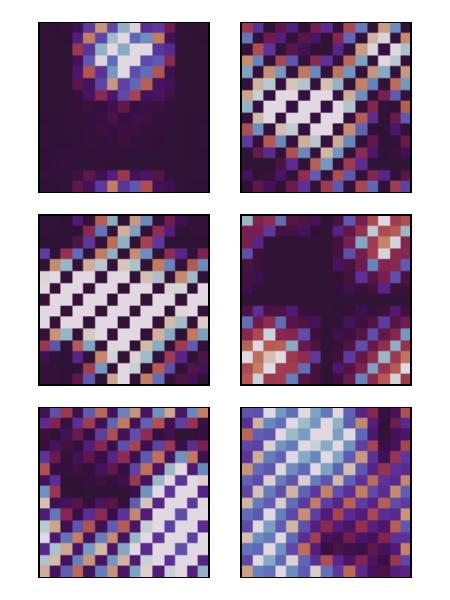}

}\hspace{-0.5cm}

\caption{\label{fig:high_freq_autoencoder} CNN ($L=10,c_\ell=60,\lambda=0.0005,\beta=0.25$) trained on images that are made up of random low-freq. shapes multiplied with a high frequency ($\omega=(5,5)$) pattern. In the bottleneck the network keeps track of the shapes in low frequencies ($\| \omega \|_1 \leq 2$) and the pattern in one $\omega=(5,5)$ frequency. Note that the original images only has signal in high frequencies around $(5,5)$.}
\end{figure}

\begin{figure}[h]\centering
\hspace{-0.2cm}\subfloat[Sing. vals. of $M W_\ell$]{\includegraphics[scale=0.33,valign=t]{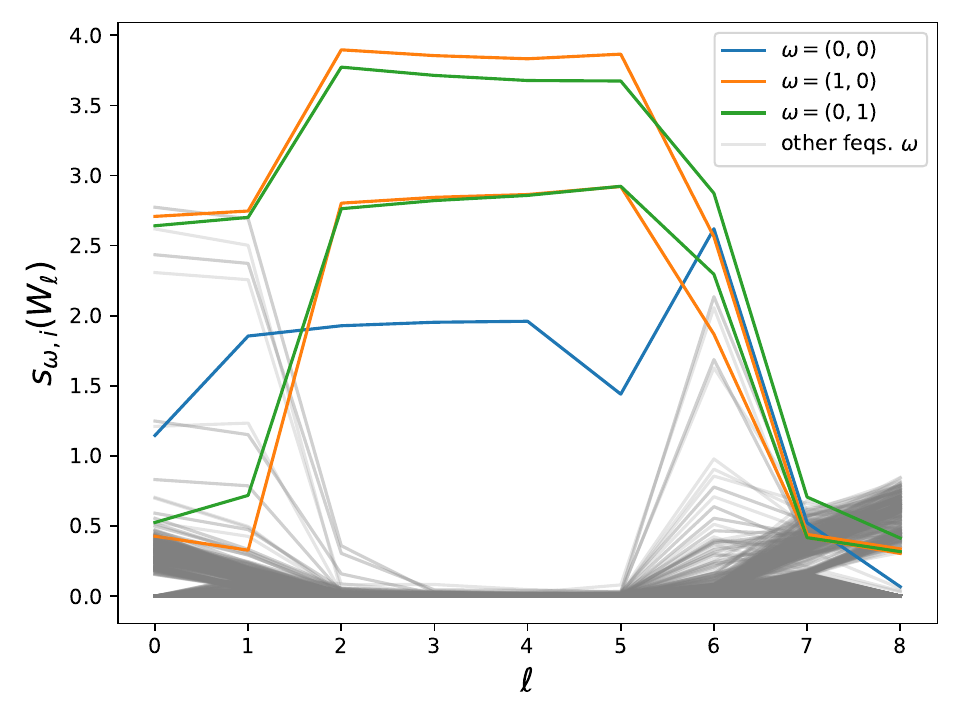}
}\;%
\subfloat[Interpretation]{\includegraphics[scale=0.35,valign=t]{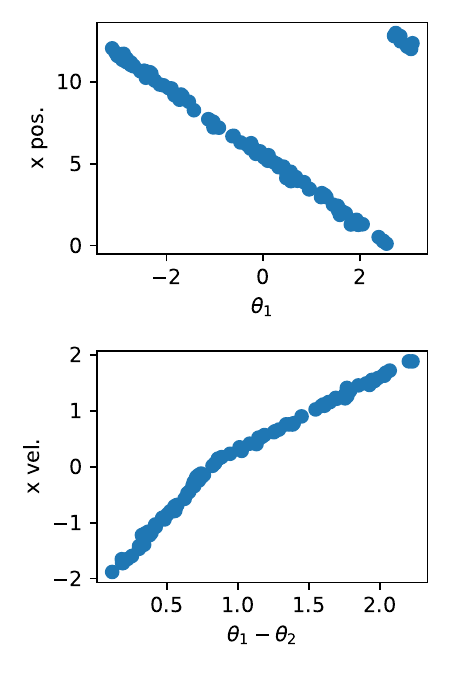}

}\hspace{-0.5cm}

\caption{\label{fig:pred_ball} CNN ($L=9,c_\ell=60,\lambda=0.0001,\beta=0.25$) learns to predict the trajectory of a ball under gravity: the inputs are 4 frames of a ball represented as a dot on a black background, and the outputs are the next four frames. The position appears to be encoded by the phase of the first pair, while the velocity is encoded in the difference between the phases of the two pairs, as confirmed in (b) along the $x$-axis.}
\end{figure}

\section{Limitations and Discussion}
In this paper we focused on describing a bottleneck structure in CNNs with small parameter norm. It still remains to be shown that GD converges under reasonable assumption to such a small parameter solution. The training dynamics of deep nonlinear networks are very difficult to study (outside of the NTK regime \cite{jacot2018neural}), but knowing what kind of structure we expect to appear will probably be helpful.

Our analysis is centered around $L_2$-regularized networks, but we expect a similar picture to appear in other settings. It has been observed that training with GD on a cross-entropy loss leads to an implicit $L_2$ regularization \cite{gunasekar_2018_implicit_bias}, thus leading to a BN structure. Similarly a small initialization should bias GD towards solutions with small parameter norm, similar to the dynamics observed in linear nets \cite{li2020towards, jacot-2021-DLN-Saddle}.

A final limitation is our use of full-size filters and cyclic boundaries. This choices has the obvious advantage of allowing for the use of Fourier analysis, but we expect a similar structure to still appear, though possibly with a different type of sparsity than the sparsity in Fourier basis that we observe. Our analysis only captures the bias induced by the translation invariance of the weights, but the locality of the connections is generally believed to also play an important role.

\section{Conclusion}
This paper describes a bottleneck structure in CNNs: the network learns functions of the form $f=h\circ g$ where the inner representation is only supported along a few Fourier frequencies, inspired by the appearance of a similar structure in fully-connected networks \cite{jacot_2022_BN_rank,jacot_2023_bottleneck2}. Our results provide motivation and justification for the common use of down-sampling in CNNs. This bottleneck structure makes the learned latent features of CNNs highly interpretable, as confirmed by a number of numerical experiments.

\section*{Impact Statement}
This work is theoretical in nature and the goal is to advance our understanding in machine learning. It has no direct societal impact. 


\bibliography{ICML2024/ConvNet_BN}
\bibliographystyle{icml2024}

\newpage
\appendix
\onecolumn
\section{CNNs as Universal Approximators}
\label{app:universality}
Since (fully connected) neural networks are mostly known as universal approximators, one may wonder if the CNNs given by the parameterization in Section~\ref{sec:paramerization} are universal approximators for translationally equivariant functions. Indeed, when the user-chosen filter $m$ is invertible, universality is guaranteed \cite{Yarotsky18}. 

Note that the filter being invertible does not prevent it from ``shrinking" the high frequencies, since it can have arbitrarily small (but nonzero) singular values at high frequencies. Thereby one may consider it as a smoothened low-pass filter.

\section{Properties of $R^{(0)}$ and $R^{(1)}$}
\label{app:properties}
We present several interesting properties of the $R^{(0)}$ cost and their proofs here.
\begin{proposition}[$R^{(0)}$ properties]\label{prop:R0_rank_properties}
    Write $\bm\coloneqq \sum_{t=1}^{n}\tilde{m}_t^{-2}$ for simplicity. For any translationally equivariant functions $f_1, f_2$, we have the following properties:
    \begin{enumerate}
        \item $R^0(f_2\circ f_1; \Omega) \leq \min\{R^0(f_2; f_1(\Omega)), R^0(f_1; \Omega)\}$;
        \item $R^0(f_2+ f_1;\Omega) \leq R^0(f_2; \Omega) +R^0(f_1; \Omega)$;
        \item $R^0(f_1; \Omega)\leq \min\{c_{in}, c_{out}\}\bm$ if the filter $m$ is nonnegative and invertible;
        \item $R^0(id; \Omega)= \min\{c_{in}, c_{out}\}\bm$ if there is an interior point of $\Omega$ in $\Omega_-$;
        \item $R^0(f_1; \Omega)= \min\{c_{in}, c_{out}\}\bm$ if $f_1$ is a bijection and there is an interior point of $\Omega$ in $\Omega_-$;
        \item If $f: \R^{n\times c_{in}}\rightarrow\R^{n\times c_{out}}$ is an affine convolution, namely $f$ is a linear one-layer CNN with weight $w\in\R^{n\times c_{out}\times c_{in}}$ and bias $b\in\R^{c_{out}}$, and there is an interior point of $\Omega$ in $\Omega_-$, then 
        \[
        R^0(f;\Omega) = \sum_{c=1}^{\min\{c_{in}, c_{out}\}}\sum_{t\in I_c}\tilde{m}_t^{-2}
        \]
        where $I_c$ is the Fourier frequencies supported by $\{w_{:,c,k}\}_{k=1}^{c_{in}}$, i.e. the indices of nonzero entries of $\{Fw_{:,c,s}\}_{k=1}^{c_{in}}$.
    \end{enumerate}
\end{proposition}
\begin{proof}
    $\boldsymbol{1.}$ Write $f_1(\Omega)\subseteq \R^{n\times c_{mid}}$. Without loss of generality, we may assume $R^0(f_2; f_1(\Omega))\leq R^0(f_1; \Omega)$. By Lemma~\ref{lemma:bounded_depth}, we can fix a CNN with depth $L_1=\ceil{nc_{in}+1}+2$ and parameter $\boldsymbol{W}_1$ representing $f_1$. For any sufficiently large $L > L_1 + \ceil{nc_{mid}+1}+2$, we have a CNN with depth $L-L_1$ and parameter $\boldsymbol{W}_2$ that represents $f_2$ with minimal representation cost, i.e. $\|\boldsymbol{W}_2\|^2 = R(f_2; L-L_1, f_1(\Omega))$. Then the norm of the composed CNN is 
    \[
    R(f_2\circ f_1; L, \Omega) \leq \|\boldsymbol{W}_1\|^2 + R(f_2;L-L_1, f_1(\Omega)).
    \]
    Dividing by $L-L_1$ and taking $L\rightarrow\infty$ gives the inequality 
    \[
    R^0(f_2\circ f_1; \Omega) \leq \min\{R^0(f_2; f_1(\Omega)), R^0(f_1; \Omega)\}.
    \]

    $\boldsymbol{2.}$ Let $f_1$ and $f_2$ be represented by CNNs with some depth $L$ and parameters $\boldsymbol{W}_1$ and $\boldsymbol{W}_2$, respectively, with minimal parameter norms, i.e. $\|\boldsymbol{W}_1\|^2=R(f_1; L, \Omega)$ and $\|\boldsymbol{W}_2\|^2=R(f_2;L,\Omega).$ We can construct a network with depth $L$ and parameters $\boldsymbol{W}$ that represents $f_1+f_2$ by stacking them "in parallel":

    At each layer $\ell$, we let the number of channels be the sum of the other two networks, i.e. $c_\ell = c^{(1)}_\ell + c^{(2)}_\ell$. For layers $1\leq \ell<L$, we set the weights and biases as follows: 
    \begin{align*}
        (w_\ell)_{:,c,k} &= \mathbbm{1}\left[c\leq c_\ell^{(1)} \land k \leq c_{\ell-1}^{(1)} \right](w^{(1)}_\ell)_{:,c,k} + \mathbbm{1}\left[c> c_\ell^{(1)} \land k > c_{\ell-1}^{(1)} \right](w^{(2)}_\ell)_{:,c-c_\ell^{(1)},k-c_{\ell-1}^{(1)}}\\
        (b_\ell)_{c} &= \mathbbm{1}\left[c\leq c_\ell^{(1)} \right](b^{(1)}_\ell)_{c} + \mathbbm{1}\left[c> c_\ell^{(1)}\right](b^{(2)}_\ell)_{c-c_\ell^{(1)}}
    \end{align*}
    for $c=1,\dots,c_\ell$ and $k=1,\dots,c_{\ell-1}$. We incorporate the sum in the last layer by having
    \begin{align*}
        (w_L)_{:,c,k} &= \mathbbm{1}\left[k\leq c_{L-1}^{(1)} \right](w^{(1)}_L)_{:,c,k} + \mathbbm{1}\left[k > c_{L-1}^{(1)} \right](w^{(2)}_L)_{:,c,k-c_{L-1}^{(1)}}\\
        (b_L)_{c} &= (b^{(1)}_L)_{c} + (b^{(2)}_L)_{c}\\
    \end{align*}
    for $c=1,\dots,c_{out}$, $k=1,\dots,c_{L-1}$.
    
    This CNN represents $f_1+f_2$ and has parameter norm $\|\boldsymbol{W}\|^2 = \|\boldsymbol{W}_1\|^2 + \|\boldsymbol{W}_2\|^2$. Hence we have
    \[
    R(f_2+f_1;L,\Omega) \leq R(f_2;L,\Omega) + R(f_1;L,\Omega).
    \]
    Dividing by $L$ and taking $L\rightarrow\infty$, we obtain the desired inequality
    \[
    R^0(f_2+ f_1;\Omega) \leq R^0(f_2; \Omega) +R^0(f_1; \Omega).
    \]

    $\boldsymbol{3.}$ Let $\phi$ be represented by a depth $L_{\phi}$ network with parameter $\bw_{\phi}$. For $L\geq L_{\phi}$, we can let the first $L-L_\phi$ layers of the network be the identity layers: Since $\Omega$ is bounded, let $K>0$ upper bounds $x\in\Omega$ coordinate-wise. For $\ell=1,\dots,L-L_\phi$, let
    \begin{align*}
        (w_{\ell})_{:,c,k} &= m^{-1}\mathbbm{1}[c=k]\\
        (b_\ell)_{c} &= K\mathbbm{1}[\ell=1]\\
        (b_{L-L_\phi+1})_c &= (b^{(\phi)}_1)_c - K(m^{-1}*\mathbf{1}).
    \end{align*}
    Note this CNN represents $\phi\circ id=\phi$. By construction, the parameter norm is $\|\boldsymbol{W}\|^2 = (L-L_{\phi})c_{in}\sum_{t=1}^{n}\tilde{m}_t^{-2} + \|\boldsymbol{W}_{\phi}\|^2 + nK(m^{-1}*\mathbf{1}+1)$. Dividing by $L$ and taking $L\rightarrow\infty$ yields
    \[
    R^0(\phi; \Omega)\leq c_{in}\sum_{t=0}^{n-1}\tilde{m}_t^{-2}.
    \]
    Similarly, appending the identity layers after $\boldsymbol{W}_\phi$ yields
    \[
    R^0(\phi; \Omega)\leq c_{out}\sum_{t=0}^{n-1}\tilde{m}_t^{-2}.
    \]

    $\boldsymbol{4.}$ Follows from the squeezing bounds Theorem~\ref{thm:CBN_upper_bound} and Theorem~\ref{thm:m_lower_bound}.

    $\boldsymbol{5.}$ Follows from the observation that 
    \[
    \min(c_{in}, c_{out})\bm=R^0(id;\Omega) \leq \min\{R^0(\psi;\Omega),R^0(\psi^{-1};\Omega)\}\leq \min(c_{in}, c_{out})\bm.
    \]

    $\boldsymbol{6.}$ Follows from the squeezing bounds Theorem~\ref{thm:CBN_upper_bound} and Theorem~\ref{thm:m_lower_bound} (note for the upper bound decomposition we have $f=id\circ f$).
\end{proof}

\begin{corollary}
    Let $f$ be any translationally equivariant function. For any translationally equivariant bijections $\phi$ and $\psi$ on $\mathbb{R}^{n\times c_{in}}$ and $\mathbb{R}^{n\times c_{out}}$ respectively, we have $R^0(\psi\circ f\circ\phi;\Omega)=R^0(f;\Omega).$
\end{corollary}

With the following proposition, we show that the $R^{(1)}$ correction controls the regularity of the learned function and satisfies subadditivity.

\begin{proposition}[$R^{(1)}$ properties]
For any translationally equivariant functions $f$ and $g$, we have the following properties:
\begin{enumerate}
    \item For any $x\in \Omega_-$, $R^1(f;\Omega)\geq 2\sum_{s_{t,c}\neq 0}\tilde{m}_t^{-2}\log (s_{t,c}\tm_t)$ with $s_{t,c}$ being the $(t,c)$-th singular values of $Jf(x)$ for $t=1,\dots,n$ and $c=1,\dots,\min\{c_{in},c_{out}\}$.

    In particular, when there is no pooling, i.e. $m=id$, for any $x\in\Omega$, we have $R^1(f;\Omega)\geq 2\log|Jf(x)|_+$ where $|\cdot|_+$ denotes the pseudo-determinant.
    \item If $R^0(f\circ g;\Omega) = R^0(f;g(\Omega)) = R^0(g;\Omega)$, then $R^1(f\circ g;\Omega) \leq R^1(f;g(\Omega))+ R^1(g;\Omega)$.
    \item If $R^0(f+g;\Omega) = R^0(f;\Omega) + R^0(g;\Omega)$, then $R^1(f+ g;\Omega) \leq R^1(f;\Omega)+ R^1(g;\Omega)$.
\end{enumerate}
\end{proposition}
\begin{proof}
    $\boldsymbol{1.}$ From the proof of Theorem~\ref{thm:m_lower_bound} we have $R(f;\Omega,L) \geq L\|M^{1-L}Jf_\theta(x)\|^{\nicefrac{2}{L}}_{\nicefrac{2}{L}}$ for any $x\in\Omega_-$. Therefore,

\begin{align*}
    R^1(f;\Omega) &= \lim_{L\rightarrow\infty}R(f;\Omega,L) - LR^0(f;\Omega)\\
    &\geq \lim_{L\rightarrow\infty}L\left(\sum_{\substack{t=1\\c=1}}^{\substack{n\\\min\{c_{in},c_{out}\}}}\tilde{m}_t^{2\frac{1-L}{L}}s_{t,c}^{\frac{2}{L}} - R^0(f;\Omega)\right)\\
    &\geq \lim_{L\rightarrow\infty}L\sum_{s_{t,c}\neq 0}\tilde{m}_t^{-2}\left(s_{t,c}^{\frac{2}{L}}\tm_t^{\frac{2}{L}} - 1\right)\\
    &\geq \lim_{L\rightarrow\infty}L\sum_{s_{t,c}\neq 0}\tilde{m}_t^{-2}\frac{2}{L}\log (s_{t,c}\tm_t)\\
    &= 2\sum_{s_{t,c}\neq 0}\tilde{m}_t^{-2}\log (s_{t,c}\tm_t)
\end{align*}
for all $x\in \Omega_-$ i.e. that are constant along each channel.

Similarly, when there is no pooling, we have for any $x\in\Omega$, $R(f;\Omega,L) \geq L\|Jf_\theta(x)\|^{\nicefrac{2}{L}}_{\nicefrac{2}{L}}$ and the result follows from the same reasoning.

$\boldsymbol{2.}$ Since $R(f\circ g;\Omega, L_1+L_2)\leq R(f;g(\Omega), L_1) + R(g;\Omega,L_2)$, we have:
\begin{align*}
    R^1(f\circ g;\Omega) &= \lim_{L_1+L_2\rightarrow\infty}R(f\circ g; \Omega, L_1+L_2) - (L_1+L_2)R^0(f\circ g; \Omega)\\
    &\leq \lim_{L_1\rightarrow\infty}R(f; g(\Omega), L_1) - L_1R^0(f; g(\Omega))+ \lim_{L_2\rightarrow\infty}R(g; \Omega, L_2) - L_2R^0(g; \Omega)\\
    &=R^1(f;g(\Omega)) + R^1(g;\Omega).
\end{align*}

$\boldsymbol{3.}$ Since $R(f+ g;\Omega, L)\leq R(f;\Omega, L) + R(g;\Omega,L)$, we have:
\begin{align*}
    R^1(f+ g;\Omega) &= \lim_{L\rightarrow\infty}R(f+ g; \Omega, L) - LR^0(f+ g; \Omega)\\
    &\leq \lim_{L\rightarrow\infty}R(f; \Omega, L) - LR^0(f; \Omega)+ \lim_{L\rightarrow\infty}R(g; \Omega, L) - LR^0(g; \Omega)\\
    &=R^1(f;\Omega) + R^1(g;\Omega).
\end{align*}
\end{proof}

\section{Upper and Lower bounds for Rescaled Representation Cost and Correction}
In this section, we present the proofs for the CBN upper bound (Theorem~\ref{thm:CBN_upper_bound}) and the filter-dependent lower bounds (Theorem~\ref{thm:m_lower_bound}) of $R^{(0)}$.
\begin{theorem}
    For any translationally equivariant function $f$ with finite $R^{(0)}(f;\Omega)$,
    \[
    R^{(0)}(f;\Omega) \leq \mathrm{Rank}_\text{CBN}(f; \Omega).
    \]
\end{theorem}
\begin{proof}
    Let $f=h\circ g$ for any TEPL functions $h,g$ and $g(x) = g_1(x)\oplus\cdots\oplus g_k(x)$, $g_c(x)$ with $I_c$ truncated Fourier coefficient supports for $c=1,\dots,k$. Lemma~\ref{lemma:bounded_depth} tells that $h$ and $g$ can be represented by CNNs with parameters $\boldsymbol{W}_h$ and $\boldsymbol{W}_g$ and depths $L_h,L_g \leq \ceil{\log (nc_{in} +1)} +2$ respectively.

    Since $\Omega$ is bounded, we can translate $g(\Omega)$ to the first quarter of $\mathbb{R}^{n\times k}$ by adding an extra bias $\Bar{b}_g$ in the last layer. Then for any $L > L_h+L_g$, we can efficiently construct a network as follows: first $L_g$ layers are the network representing $g$ with an extra bias $\Bar{b}_g$ to translate the output to the first quarter, followed by $L-L_h-L_g$ identity layers as described in the proof of Proposition~\ref{prop:R0_rank_properties} and translate the hidden representation $\alpha_\ell$ back by subtracting $\Bar{b}_g$ in the last identity layer, and finally the last $L_h$ layers the network representing $h$. This construction gives us a bound
    \[
    R(f; L, \Omega) \leq \|\boldsymbol{W}_g\|^2 + (L-L_g-L_h)\sum_{c=1}^k\sum_{i\in I_c}\tilde{m}_i^{-2} + 2\|\Bar{b}_g\|^2 + \|\boldsymbol{W}_h\|^2
    \]
    for $L>L_g+L_h$. Dividing both side by $L$ and taking $L\rightarrow\infty$ gives the inequality 
    \[
    R^0(f; \Omega) \leq \sum_{c=1}^k\sum_{i\in I_c}\tilde{m}_i^{-2}
    \]
    and the result follows since $h\circ g$ is an arbitrary TEPL decomposition.
\end{proof}

\begin{theorem}
    For any translationally equivariant function $f$, let $Jf(x)$ be the Jacobian of $f$ at $x$. The following filter-dependent lower bounds hold:
    \begin{enumerate}
        \item $\frac{1}{\max\{\tm^2_{\max},1\}}\max_{x\in\Omega}\mathrm{Rank}(Jf(x)) \leq R^{(0)}(f;\Omega)$
        
        In particular, when there is no pooling, $\max_{x\in\Omega}\mathrm{Rank}(Jf(x)) \leq R^{(0)}(f;\Omega)$.
        
        \item $\max_{x\in\Omega_-}\mathrm{Rank}_m(Jf(x)) \leq R^{(0)}(f;\Omega)$ 
        
        where the max is now taken over the subset $\Omega_-\coloneqq\{x\in\Omega\,|\, x_{i,c}=x_{j,c}\ \forall c=1,\dots,c_{in}, \forall i,j=1,\dots,n \}$, i.e. all $x$ that are constant along each channel.
    \end{enumerate}
\end{theorem}
\begin{proof}
    $\boldsymbol{1.}$ Fix any input $x\in\Omega$, depth $L$, and the minimal-norm parameter $\theta$ with $f_\theta=f$. We can first write
    \[
    Jf_{\theta}(x) = W_LD_{L-1}(x)MW_{L-1}\cdots D_1(x)MW_1
    \]
    where $D_\ell(x) = \mathrm{diag}(\Dot{\sigma}(\alpha_\ell(x)))\in\R^{nc_\ell\times nc_{\ell}}$ are diagonal matrices with $1$ and $0$ on the diagonal, $W_\ell\in\R^{nc_\ell\times nc_{\ell-1}}$ are the matrix representation of the convolution filters $w_\ell\in\R^{n\times c_\ell\times c_{\ell-1}}$, and $M$ is that of the channel-wise convolution with $m\in\R^n$ (or simply a convolution filter $\hat{m}\in\R^{n\times c_\ell\times c_\ell}$ with $\hat{m}_{:,c,s} = \mathbbm{1}[c=s]m$). From \cite{jacot_2022_BN_rank} and \cite{dai_2021_repres_cost_DLN}, we have
    \begin{align*}
    \left\Vert Jf_{\theta}(x)\right\Vert _{\nicefrac{2}{L}}^{\nicefrac{2}{L}} 
    & \leq\frac{1}{L}\left(\left\Vert W_{L}\right\Vert _{F}^{2}+\left\Vert D_{L-1}(x)MW_{L-1}\right\Vert _{F}^{2}+\dots+\left\Vert D_{1}(x)MW_{1}\right\Vert _{F}^{2}\right)\\
     & \leq\frac{1}{L}\left(\left\Vert W_{L}\right\Vert _{F}^{2} + \left\Vert MW_{L-1}\right\Vert _{F}^{2}+\dots+\left\Vert MW_{1}\right\Vert _{F}^{2}\right)\\
     & \leq\frac{1}{L}\max\{\|M\|^2_{2},1\}\left(\left\Vert W_{L}\right\Vert _{F}^{2}+\dots+\left\Vert W_{1}\right\Vert _{F}^{2}\right)\\
     & =\frac{\max\{\tm^2_{\max},1\}}{L}\left(\left\Vert W_{L}\right\Vert _{F}^{2}+\dots+\left\Vert W_{1}\right\Vert _{F}^{2}\right)\\
     & \leq \max\{\tm^2_{\max},1\}R(f;\Omega,L)/L
    \end{align*}

    Taking $L\rightarrow\infty$ on both sides, we have for any input $x\in\Omega$,
    \[
    \frac{1}{\max\{\tm^2_{\max},1\}}\mathrm{Rank}(Jf(x)) \leq R^{(0)}(f;\Omega).
    \]

    $\boldsymbol{2.}$ The key observation in this proof is that if the input $x$ is constant along each channel and $W$ is any translationally equivariant matrix, then $WD_\ell(x)M = MWD_\ell(x)$. This observation follows from the fact that $\alpha_\ell(x)$ is translationally equivariant and hence also channel-wise constant; then $D_\ell(x)=\mathrm{diag}(\Dot{\sigma}(\alpha_\ell(x)))$ is either all $0$ or all $1$ along each channel, and so $D_\ell(x)M=MD_\ell(x)$. The commutativity between $M$ and $W$ always holds and follows from pure algebraic computation. Consequantly, for any $x\in\Omega_-$, depth $L$, and parameter $\theta$ with $f_\theta =f$, we now have
    \begin{align*}
        Jf_{\theta}(x) &= W_LD_{L-1}(x)MW_{L-1}\cdots D_1(x)MW_1\\
        &= M^{L-1}W_LD_{L-1}(x)W_{L-1}\cdots D_1(x)W_1
    \end{align*}
    and
    \begin{align*}
        \left\Vert M^{1-L}Jf_{\theta}(x)\right\Vert _{\nicefrac{2}{L}}^{\nicefrac{2}{L}}
        &= \left\Vert W_LD_{L-1}(x)W_{L-1}\cdots D_1(x)W_1\right\Vert_{\nicefrac{2}{L}}^{\nicefrac{2}{L}}\\
        & \leq\frac{1}{L}\left(\left\Vert W_{L}\right\Vert _{F}^{2}+\left\Vert D_{L-1}(x)W_{L-1}\right\Vert _{F}^{2}+\dots+\left\Vert D_{1}(x)W_{1}\right\Vert _{F}^{2}\right)\\
        &\leq \frac{1}{L}\left(\left\Vert W_{L}\right\Vert _{F}^{2}+\dots+\left\Vert W_{1}\right\Vert _{F}^{2}\right)\\
        &\leq R(f;L,\Omega)/L \numberthis \label{eq:filter_jacobian_bound}
    \end{align*}
    
    Since $Jf_{\theta}(x)$ is a product of translationally equivariant matrices and so is translationally equivariant, we can index its
    singular values $s_{t,c}$ by the FDT frequency $t=0,\dots,n-1$ and the channel $c=1,\dots,\min\{c_{in},c_{out}\}$. Then the singular values
    of $M^{(1-L)}Jf_{\theta}(x)$ are $\tilde{m}_{t}^{1-L}s_{t,c}$.
    Thus we can rewrite
    \[
    \left\Vert M^{1-L}Jf_{\theta}(x)\right\Vert _{\nicefrac{2}{L}}^{\nicefrac{2}{L}}=\sum_{t=1}^{n}\sum_{c=1}^{\min\{c_{in},c_{out}\}}\tilde{m}_{t}^{2\frac{1-L}{L}}s_{t,c}^{\nicefrac{2}{L}}.
    \]
    Taking $L\to\infty$ on both sides of (\ref{eq:filter_jacobian_bound}), we have
    \[
    \sum_{t=1}^{n}\sum_{c=1}^{\min\{c_{in},c_{out}\}}\tilde{m}_{t}^{-2}\mathbbm{1}[s_{t,c}\neq 0] \leq R^{(0)}(f;\Omega).
    \]
\end{proof}

\section{Bottleneck Structure in Weights and Activations}
Following are the proofs for Theorem~\ref{thm:bottleneck_weights} and Theorem~\ref{thm:bounded_activations}.
\begin{theorem}
    Given $\|\theta\|^2\leq L\max_{z\in\Omega_-}\mathrm{Rank}_{m}(Jf_\theta(z)) + c_1$ and $x\in \mathrm{argmax}_{z\in \Omega_-}\mathrm{Rank}_mJf_\theta(x)$, we have $V_\ell^T\in\R^{\kappa\times nc_{\ell-1}}$ and $U_\ell\in\R^{nc_\ell\times\kappa}$ being submatrices of the DFT block matrices $F_{\ell-1}\in\R^{nc_{\ell-1}\times nc_{\ell-1}}$ and $F_\ell^*\in\R^{nc_{\ell}\times nc_{\ell}}$ respectively, where $\kappa = \mathrm{Rank}Jf_\theta(x)$, such that 
    \[
    \sum_{\ell=1}^L\|W_\ell-U_{\ell} S_\ell V_\ell^T\|^2_F + \|b_\ell\|_F^2\leq c_1 - 2\sum_{s_{t,c}\neq 0}\tm_t^{-2}\log (s_{t,c}\tm_t)
    \]
    and thus for any $p\in (0,1)$, there are at least $(1-p)L$ layers $\ell$ with
    \[
    \|W_\ell-U_{\ell} S_\ell V_\ell^T\|^2_F + \|b_\ell\|_F^2\leq \frac{c_1 - 2\sum_{s_{t,c}\neq 0}\tm_t^{-2}\log (s_{t,c}\tm_t)}{pL}
    \]
    where $s_{t,c}$ is the $(t,c)$-th singular value of $Jf_\theta(x)$ and $S_\ell\in\R^{\kappa\times\kappa}$ is a diagonal matrix with entries $\in\{\tm_t^{-1}\}_{t=0}^{n-1}$.
\end{theorem}
\begin{proof}
    Note for $x\in \Omega_-$ constant input, we have
    \begin{align*}
        Jf_\theta(x) &= W_LD_{L-1}(x)MW_{L-1}\cdots D_1(x)MW_1
    \end{align*}
    where all $W_\ell$, $D_\ell(x)$, and $M$ are translationally equivariant, and so is $Jf_\theta(x)$ itself. Hence we can decompose $Jf_\theta(x)$ along each Fourier frequency separately: Let $P_t$ be the projection matrix to the signal space consisting of only the $t$-th frequency. Then we can decompose
    \[
    Jf_\theta(x) = \sum_{t=1}^{n}P_tJf_\theta(x)P_t\,.
    \]

    Now we consider each summand $P_tJf_\theta(x)P_t$. For $x\in \Omega_-$, each $P_t$ represents a single-channel convolution and hence commutes with all translationally equivariant maps. Then we have
    \begin{align*}
        P_tJf_\theta(x)P_t &= P_tW_LD_{L-1}(x)MW_{L-1}\cdots D_1(x)MW_1P_k\\
        &= P_tW_LP_tD_{L-1}(x)P_tMP_tW_{L-1}P_t\cdots D_1(x)P_tMP_tW_1P_t\\
        &= \tm_t^{L-1}P_tW_LP_tD_{L-1}(x)P_tW_{L-1}P_t\cdots D_1(x)P_tW_1P_t\\
        &= \tm_t^{L-1}P_tW_LP_tP_{\Imm J\alpha_{L-1}(x)}P_{\Imm J(\alpha_{L-1}\rightarrow f_\theta)(x)^T}D_{L-1}(x)P_{\Imm J\talpha_{L-1}(x)}\\
        &\indent P_{\Imm J(\talpha_{L-1}\rightarrow f_\theta)(x)^T}P_tW_{L-1}P_tP_{\Imm J\alpha_{L-2}(x)}\cdots \\
        &\indent P_{\Imm J(\alpha_{1}\rightarrow f_\theta)(x)^T}D_1(x)P_{\Imm J\talpha_{1}(x)}P_{\Imm J(\talpha_{1}\rightarrow f_\theta)(x)^T}P_tW_1P_t
    \end{align*}
    since $P_tMP_t=\tm_tP_t$.
    
    For general matrices $A$ and $B$, $|AB|_+=|A|_+|B|_+$ when the non-zero pre-image of $A$ matches the image of $B$, and $|\alpha A|_+ = \alpha^{\mathrm{Rank}A}|A|_+$. Hence we have
    \begin{align*}
        |P_tJf_\theta(x)P_t|_+ &= \tm_t^{(L-1)n_t}|P_tW_LP_tP_{\Imm J\alpha_{L-1}(x)}|_+|P_{\Imm J(\alpha_{L-1}\rightarrow f_\theta)(x)^T}D_{L-1}(x)P_{\Imm J\talpha_{L-1}(x)}|_+\\
        &\indent |P_{\Imm J(\talpha_{L-1}\rightarrow f_\theta)(x)^T}P_tW_{L-1}P_tP_{\Imm J\alpha_{L-2}(x)}|_+\cdots \\
        &\indent |P_{\Imm J(\alpha_{1}\rightarrow f_\theta)(x)^T}D_1(x)P_{\Imm J\talpha_{1}(x)}|_+|P_{\Imm J(\talpha_{1}\rightarrow f_\theta)(x)^T}P_tW_1P_t|_+
    \end{align*}
    with $n_t = \sum_{c=1}^{\min\{c_{in}, c_{out}\}}\mathbbm{1}[s_{t,c}\neq 0] = \mathrm{Rank}(P_tJf_\theta(x)P_t)$. Then writing $P_{\Imm J\alpha_{0}(x)} = I$, we have
    \begin{align*}
        \sum_{\substack{c=1\\s_{t,c}\neq 0}}^{\min\{c_{in}, c_{out}\}}\log s_{t,c} &= \log|P_tJf_\theta(x)P_t|_+\\
        &= n_t(L-1)\log\tm_t + \sum_{\ell=1}^{L-1}|P_{\Imm J(\alpha_{\ell}\rightarrow f_\theta)(x)^T}D_{\ell}(x)P_{\Imm J\talpha_{\ell}(x)}|_+\\
        &\indent + \sum_{\ell=1}^L|P_{\Imm J(\talpha_{\ell}\rightarrow f_\theta)(x)^T}P_tW_{\ell}P_tP_{\Imm J\alpha_{\ell-1}(x)}|_+ \,.
    \end{align*}
    Observe that 
    \[
    -2\tm_t^{-2}|P_{\Imm J(\alpha_{\ell}\rightarrow f_\theta)(x)^T}D_{\ell}(x)P_{\Imm J\talpha_{\ell}(x)}|_+ \geq \tm_t^{-2}\left(\mathrm{Rank}Jf_\theta(x) - \|P_{\Imm J(\alpha_{\ell}\rightarrow f_\theta)(x)^T}D_{\ell}(x)P_{\Imm J\talpha_{\ell}(x)}\|^2_F\right)
    \]
    which is positive since the eigenvalues of $D_{\ell}(x)$ is $\leq 1$.
    
    Also note that for general matrix $A$ and constants $m_i$, we have
    \begin{align*}
        &\|A\|_F^2 - \sum_{i=1}^{\mathrm{Rank}A} \left(m_i^{-2} -2m_i^{-2}\log m_i - 2m_i^{-2}\log s_i(A)\right)\\
        &= \sum_{i=1}^{\mathrm{Rank}A}s_i(A)^2-m_i^{-2}(1+2\log m_i +2\log s_i(A))\\
        &= \sum_{i=1}^{\mathrm{Rank}A}s_i(A)^2-m_i^{-2}(1+2\log (m_is_i(A)))\\
        &\geq \sum_{i=1}^{\mathrm{Rank}A}s_i(A)^2-m_i^{-2}(1+2m_i s_i(A) - 2)\\
        &= \sum_{i=1}^{\mathrm{Rank}A}(s_i(A)- m_i^{-1})^2
    \end{align*}
    Denote $\overline{W}_\ell^{(t)} = P_{\Imm J(\talpha_{\ell}\rightarrow f_\theta)(x)^T}P_tW_{\ell}P_tP_{\Imm J\alpha_{\ell-1}(x)}$ for simplicity. Then we can lower bound the sum
    \begin{align*}
        &\|\theta\|^2 - \sum_{\ell=1}^L\|b_\ell\|_F^2 - L\mathrm{Rank}_{m}(Jf_\theta(x)) - 2\sum_{\substack{t,c\\s_{t,c}\neq 0}}\tm_t^{-2}\log(s_{t,c}\tm_t)\\
        &=\sum_{\ell=1}^L\|W_\ell\|_F^2 - L\sum_{\substack{t,c\\s_{t,c}\neq 0}}\tm_t^{-2} - 2\sum_{\substack{t,c\\s_{t,c}\neq 0}}\tm_{t}^{-2}\log (s_{t,c}\tm_t) \\
        &=\sum_{\ell=1}^L\|W_\ell\|_F^2 - L\sum_{\substack{t,c\\s_{t,c}\neq 0}}\tm_t^{-2} - 2\sum_{\substack{t,c\\s_{t,c}\neq 0}}\tm_{t}^{-2}\log s_{t,c} - 2\sum_{\substack{t,c\\s_{t,c}\neq 0}}\tm_{t}^{-2}\log \tm_t \\
        &\geq \sum_{\ell=1}^L\sum_{t=1}^{n}\Bigg(\|P_tW_\ell P_t\|_F^2 - \sum_{\substack{c\\s_{t,c}\neq 0}}\tm_t^{-2} - 2\tm_{t}^{-2}\sum_{\substack{c\\s_{t,c}\neq 0}}\log\tm_t - 2\tm_{t}^{-2}\log|\overline{W}_\ell^{(t)}|_+\Bigg)\\
        &= \sum_{\ell=1}^L\sum_{t=1}^{n}\Bigg(\|P_tW_\ell P_t - \overline{W}_\ell^{(t)}\|_F^2 + \|\overline{W}_\ell^{(t)}\|_F^2 - \sum_{\substack{c\\s_{t,c}\neq 0}}\tm_t^{-2} - 2\tm_{t}^{-2}\sum_{\substack{c\\s_{t,c}\neq 0}}\log\tm_t - 2\tm_{t}^{-2}\log|\overline{W}_\ell^{(t)}|_+\Bigg)\\
        &\geq \sum_{\ell=1}^L\sum_{t=1}^{n}\bigg(\|P_tW_\ell P_t - \overline{W}_\ell^{(t)}\|_F^2 + \sum_{s_{t,c}(\overline{W}_\ell^{(t)})\neq 0}\left(s_{t,c}(\overline{W}_\ell^{(t)}) - \tm_t^{-1}\right)^2\bigg)\\
        &\geq \sum_{\ell=1}^L\sum_{t=1}^{n} \|P_tW_\ell P_t - U^{(t)}_\ell S^{(t)}_{\ell}(V^{(t)}_\ell)^T\|_F^2\\
        &= \sum_{\ell=1}^L\|W_\ell - U_\ell S_{\ell}V_\ell^T\|_F^2
    \end{align*}
    since
    \begin{align*}
        \sum_{t=1}^{n}\|P_tW_\ell P_t - P_{\Imm J(\talpha_{\ell}\rightarrow f_\theta)(x)^T}P_tW_{\ell}P_tP_{\Imm J\alpha_{\ell-1}(x)}\|_F^2
        &= \sum_{t=1}^{n}\|P_tW_\ell P_t - P_tP_{\Imm J(\talpha_{\ell}\rightarrow f_\theta)(x)^T}W_{\ell}P_{\Imm J\alpha_{\ell-1}(x)}P_t\|_F^2\\
        &= \|W_\ell - P_{\Imm J(\talpha_{\ell}\rightarrow f_\theta)(x)^T}W_{\ell}P_{\Imm J\alpha_{\ell-1}(x)}\|_F^2.
    \end{align*}
    Here $U_\ell \Sigma_\ell V_\ell^T$ is the compact SVD decomposition of $P_{\Imm J(\talpha_{\ell}\rightarrow f_\theta)(x)^T}W_{\ell}P_{\Imm J\alpha_{\ell-1}(x)}$. Since we know $P_{\Imm J(\talpha_{\ell}\rightarrow f_\theta)(x)^T}W_{\ell}P_{\Imm J\alpha_{\ell-1}(x)}$ is translationally equivariant, we can let $V_\ell^T\in\R^{\kappa\times nc_{\ell-1}}$ and $U_\ell\in\R^{nc_\ell\times\kappa}$ be submatrices of the DFT block matrices $F_{\ell-1}\in\R^{nc_{\ell-1}\times nc_{\ell-1}}$ and $F_\ell^*\in\R^{nc_{\ell}\times nc_{\ell}}$ respectively, and $\Sigma_\ell$ correspond to the nonzero singular values of $\kappa = \mathrm{Rank}Jf_\theta(x)$ frequencies. And $S_{\ell}\in\mathbb{R}^{\kappa\times \kappa}$ consists of the singular values of $M^{-1}$ at corresponding frequencies. 
    This gives
    \begin{align*}
        \sum_{\ell=1}^L \|W_\ell - U_{\ell} S_\ell V_\ell^T\|_F^2 + \|b_\ell\|_F^2 
        &\leq  \|\theta\|^2 - L\mathrm{Rank}_{m}(Jf_\theta(x)) - 2\sum_{s_{t,c}\neq 0}\tm_t^{-2}\log (s_{t,c}\tm_t)\\
        &\leq c_1 - 2\sum_{s_{t,c}\neq 0}\tm_t^{-2}\log (s_{t,c}\tm_t)\,.
    \end{align*}
\end{proof}

\begin{theorem}
    Given a depth $L$ network, balanced parameters $\theta$ with $\|\theta\|^2\leq L\max_{x\in\Omega_-}\rank_m(Jf_\theta(z)) + c_1$, and a point $x_0$ with $\rank Jf_\theta(x_0)=k$, then $\|J_\theta f_\theta(x_0)\|_F^2 \leq cL$ implies that,
    \[
    \sum_{\ell=1}^L\|\alpha_{\ell-1}(x_0)\|_2^2 \leq \frac{ce^{\frac{c_1}{k}}}{k|Jf_\theta(x_0)|_+^{\nicefrac{2}{k}}}L
    \]
    Hence for each $p\in(0,1)$, there are at least $(1-p)L$ layers $\ell$ with
    \[
    \|\alpha_{\ell-1}(x_0)\|_2^2 \leq \frac{1}{p}\frac{ce^{\frac{c_1}{k}}}{k|Jf_\theta(x_0)|_+^{\nicefrac{2}{k}}}\,.
    \]
\end{theorem}
\begin{proof}
    We can write
    \[
    \|J_\theta f_\theta(x_0)\|_F^2 = \Tr\left[\Theta^{(L)}(x_0,x_0) \right] = \sum_{\ell=1}(\|\alpha_{\ell-1}(x_0)\|^2_2+1)\|J(\talpha_\ell\rightarrow\alpha_L)(x_0)\|_F^2
    \]
    and our goal is to lower bound $\|J(\talpha_\ell\rightarrow\alpha_L)(x_0)\|_F^2$ for each $\ell$ in the following. We first note that $\mathrm{Rank}(J(\talpha_\ell\rightarrow\alpha_L)(x_0) P_\ell) = \rank J f(x_0) = k$ where $P_\ell$ is the projection matrix to the image of $J\talpha_\ell(x_0)$.

    By AM-GM inequality,
    \begin{align*}
    \|J(\talpha_\ell\rightarrow\alpha_L)(x_0)P_\ell\|_F^2 &\geq k|J(\talpha_\ell\rightarrow\alpha_L)(x_0)P_\ell|_+^{\nicefrac{2}{k}}\,.
    \end{align*}

    Since the parameters are balanced, i.e. $\|w_\ell\|_F^2 + \|b_\ell\|_F^2=\|w_{\ell+1}\|_F^2$, we have increasing parameter norms $\|W_\ell\|_F^2\leq \|W_{\ell+1}\|_F^2$ and so
    \[
    \frac{1}{\ell}\sum_{j=1}^\ell\|W_j\|_F^2 \leq \frac{1}{L-\ell}\sum_{j=\ell+1}^L\|W_j\|_F^2\,.
    \]
    Thus
    \begin{align*}
        \frac{1}{\ell}\sum_{j=1}^\ell\|W_j\|_F^2 &= \frac{1}{L}\sum_{j=1}^\ell\|W_j\|_F^2 + \frac{L-\ell}{L}\frac{1}{\ell}\sum_{j=1}^\ell\|W_j\|_F^2\\
        &\leq \frac{\|\theta\|^2}{L}
    \end{align*}
    and again by AM-GM inequality,
    \begin{align*}
        |P'_\ell J\talpha_\ell(x_0)|_+^{\nicefrac{2}{kL}} &\leq \frac{1}{k}\|P'_\ell J\talpha_\ell(x_0)\|_{\nicefrac{2}{L}}^{\nicefrac{2}{L}}\\
        &\leq \frac{1}{k}\frac{\|P'_\ell MW_\ell\|_F^2 + \cdots + \|MW_1\|_F^2}{\ell}\\
        &\leq \frac{1}{k}\frac{\|MW_\ell\|_F^2 + \cdots + \|MW_1\|_F^2}{\ell}\\
        &\leq \frac{\tm_{\max}}{k}\frac{\|W_\ell\|_F^2 + \cdots + \|W_1\|_F^2}{\ell}\\
        &\leq \frac{\tm_{\max}\|\theta\|^2}{kL}\\
        &\leq \frac{\tm_{\max}\max_{z\in\Omega_-}\mathrm{Rank}_{m}(Jf_\theta(z))}{k}\left(1 + \frac{c_1}{L\max_{z\in\Omega_-}\mathrm{Rank}_{m}(Jf_\theta(z))}\right)
    \end{align*}
    where $P'_\ell$ denotes the projection matrix to the image of $J(\talpha_\ell\rightarrow\alpha_L)(x_0)$. For simplicity, we denote $R = \max_{z\in\Omega_-}\mathrm{Rank}_{m}(Jf_\theta(z))$. Therefore, we have
    \begin{align*}
        \|J(\talpha_\ell\rightarrow f_\theta)(x_0)P_\ell\|_F^2 &\geq k|J(\talpha_\ell\rightarrow f_\theta)(x_0)P_\ell|_+^{\nicefrac{2}{k}}\\
        &= k\frac{|Jf_\theta(x_0)|_+^{\nicefrac{2}{k}}}{|J\talpha_\ell(x_0)|_+^{\nicefrac{2}{k}}}\\
        &\geq k\frac{|Jf_\theta(x_0)|_+^{\nicefrac{2}{k}}}{(\tm_{\max}R/k)^L\left(1+\frac{c_1}{LR}\right)^L}\\
        &\geq k|Jf_\theta(x_0)|_+^{\nicefrac{2}{k}} e^{-\left(\frac{c_1}{R} + L(\frac{\tm_{\max}R}{k}-1)\right)}\\
        &= k|Jf_\theta(x_0)|_+^{\nicefrac{2}{k}} e^{-\frac{c_1}{R}}e^{-L(\frac{\tm_{\max}R}{k}-1)}
    \end{align*}
    and hence
    \begin{align*}
        \sum_{\ell=1}^L\|\alpha_{\ell-1}(x_0)\|_2^2 &\leq \frac{ce^{\frac{c_1}{R}}e^{L(\frac{\tm_{\max}R}{k}-1)}}{k|Jf_\theta(x_0)|_+^{\nicefrac{2}{k}}}L
    \end{align*}
    which implies that for each $p\in(0,1)$, there are at most $pL$ layers $\ell$ with
    \[
    \|\alpha_{\ell-1}(x_0)\|_2^2 \geq \frac{1}{p}\frac{ce^{\frac{c_1}{R}}e^{L(\frac{\tm_{\max}R}{k}-1)}}{k|Jf_\theta(x_0)|_+^{\nicefrac{2}{k}}}\,.
    \]
\end{proof}

\begin{corollary}
    When there is no pooling, $x_0$ maximizes the rank $\mathrm{Rank}Jf(x_0)$, and the above conditions still hold, we have 
    \[
    \sum_{\ell=1}^L\|\alpha_{\ell-1}(x_0)\|_2^2 \leq \frac{ce^{\frac{c_1}{k}}}{k|Jf_\theta(x_0)|_+^{\nicefrac{2}{k}}}L.
    \]
    Hence for each $p\in(0,1)$, there are at least $(1-p)L$ layers $\ell$ with
    \[
    \|\alpha_{\ell-1}(x_0)\|_2^2 \leq \frac{1}{p}\frac{ce^{\frac{c_1}{k}}}{k|Jf_\theta(x_0)|_+^{\nicefrac{2}{k}}}\,.
    \]
\end{corollary}

\section{CNNs with Up-sampling and Down-sampling}
Here we present the proofs for characterizing all functions that can be represented by $s$-stride-CNNs as well as finding a $2$-frequency decomposition for translationally unique domains (Theorem~\ref{thm:translationally_unique_domain}).
\begin{proposition}\label{prop:Ns_decomposition}
    Any $f\in\mathcal{N}^{(s)}_{n;m,m'}$ if and only if $f$ has a low-frequency decomposition, i.e. $f = h^{(s)}\circ g^{(s)}$ where $g^{(s)},h^{(s)}$ are $s$-translationally equivariant piece-wise linear ($s$-TEPL) functions, $g^{(s)}=g^{(s)}_1\oplus\cdots\oplus g^{(s)}_k$, and $g^{(s)}_i$ only supports the first $\frac{n}{s}$ frequencies for $i=1,
    \dots,k$.
\end{proposition}
\begin{proof}
    $(\impliedby)$ Note $f = h^{(s)}\circ g^{(s)} = h^{(s)}\circ \mathrm{Up}_s \circ id_{\Imm g^{(s)}} \circ\mathrm{Down}_s\circ g^{(s)}\in\mathcal{N}^{(s)}_{n;m,m'}$, where $id_{\Imm g^{(s)}}$ can be the identity layer as we constructed before. 
    
    $(\implies)$ To see the other direction, observe that $f\in\mathcal{N}^{(s)}_{n;m,m'}$ gives $f= h^{(s)}\circ (\mathrm{Up}_s \circ \hat{f} \circ\mathrm{Down}_s\circ g^{(s)})$ where the latter is a low-frequency $s$-TEPL function.
\end{proof}

\begin{theorem}
    Suppose $\Omega$ is translationally unique. Then for any piecewise linear target function $f:\Omega\rightarrow\R^{n\times c_{out}}$, $f=h\circ g^{low}$ where $h$ and $g^{low}$ are TEPL functions and $g^{low}:\Omega\rightarrow\R^{n\times nc_{in}+1}$ only supports the constant DFT frequency at first $nc_{in}$ channels and the second DFT frequency at the $nc_{in}+1$-th channel.
\end{theorem}
\begin{proof}
    Let $\overline{\Omega} = \{T_px : x\in\Omega, p=0,\dots,n-1\}$ be the translational closure of the domain $\Omega$. Since $\Omega$ is bounded, so is $\overline{\Omega}$, and hence without loss of generality, we may assume $\overline{\Omega}$ lies in the first quarter and is upper bounded by $Z\geq 1$ coordinate-wise. By Lemma~\ref{lemma:bounded_depth}, it suffices to show there exists a TEPL function 
    \[F:\overline{\Omega}\xrightarrow{G}\left(\Omega\times \{\cos(2\pi p/n)\}_{p=0}^{n-1}\right)^n\xrightarrow{H}\R^{n\times c_{out}}
    \]
    such that $F|_\Omega = f$ and $F=H\circ G$ where $G,H$ are TEPL and $G$ has a low-frequency support at each channel. Define $G$ and $H$ as follows:
    \begin{align*}
        G(T_px)_{i,0:nc_{in}-1} &= \mathrm{vec}(x)\\
        G(T_px)_{i,nc_{in}} &= \cos(2\pi(p-i)/n)\\
        H(G(T_px)) &= T_p f(x)
    \end{align*}
    for $i\in[n]$.
    Translational equivariance follows directly from the definition; it remains to verify that $G$ and $H$ are piecewise linear. We first show $G^{-1}$ is piecewise linear by showing it can be represented by a $3$-layer no-pooling ConvNet:

    \textbf{(First layer).} Denote a threshold $\epsilon = \max_{i\neq p}\cos(2\pi(p-i)/n)$; note $\epsilon < 1$. Let $(w_1)_{i,c,s} = \delta_{i=0}\delta_{c=s}$ for $c\in[nc_{in}]$, where $\delta$ is the indicator function (so $w\oast x = x$). Let $(b_1)_c = -\epsilon \delta_{c=nc_{in}}$ for $c\in[nc_{in}]$. After applying the ReLU, we have activation
    \begin{align*}
        \alpha_1(G(T_px))_{i,0:nc_{in}-1} &= G(T_px)_{i,0:nc_{in}-1} \\
        \alpha_1(G(T_px))_{i,nc_{in}} &= \delta_{i= p}(1-\epsilon)
    \end{align*}

    \textbf{(Second layer).} Let $(w_2)_{i,c,s} = \delta_{i=0}\delta_{c=s}$ for $c\in[nc_{in}-1]$ and $\delta_{i=0}Z/(1-\epsilon)$ for $c=nc_{in}$, for $s\in[nc_{in}]$, where $\delta$ is the indicator function. Let $(b_2)_c = -Z$ for $c\in[nc_{in}]$. Then we have
    \begin{align*}
        \talpha_2(G(T_px))_{i,0:nc_{in}-1} &= G(T_px)_{i,0:nc_{in}-1}-\delta_{i\neq p}Z \\
        \talpha_2(G(T_px))_{i,nc_{in}} &= -\delta_{i\neq p}Z\\
        \alpha_2(G(T_px))_{i,0:nc_{in}-1} &= \delta_{i=p}G(T_px)_{i,0:nc_{in}-1} \\
        &= \delta_{i=p}\mathrm{vec}(x)\\
        \alpha_2(G(T_px))_{i,nc_{in}} &= 0
    \end{align*}
    since $x$ is coordinate-wise upper bounded by $Z>0$.

    \textbf{(Third layer).} Let $(w_3)_{i,c,s}=\delta_{i=c\mod n}\delta_{c=\floor{s/n}}$ and $(b_3)_c=0$ for channel $c\in[c_{in}]$ and $s\in[nc_{in}]$, i.e. translating the $s$-th channel of the input by $s\mod n$ and then summing every $n$ channels. Then we have the output being
    \begin{align*}
        \alpha_3(G(T_px)) &= T_px.
    \end{align*}

    Hence $G^{-1}$ is TEPL and $G=(G^{-1})^{-1}$ is TEPL.

    One can see $H=f\circ G^{-1}$ is also TEPL. By letting $g^{low} = G|_\Omega$ and $h=H|_{\Imm g^{low}}$, we see they are TEPL and each channel of $g^{low}$ only either supports the first or the second DFT coefficient.
\end{proof}

\section{Representation Cost in Filter Norm}
\label{app:filter_norm}
If one considers the representation cost $\|\tilde\theta\|^2$ as the norm of the filters (as opposed to the matrices in Section~\ref{sec:rep_cost}) and the biases, by definition one has $\|\tilde\theta\|^2 = \frac{1}{n}\|\theta\|^2$ off by a factor of $\frac{1}{n}$. One can then adapt the results and proofs in this paper to the filter norm by substituting $\|\tilde\theta\|^2 = \frac{1}{n}\|\theta\|^2$ and get this extra factor in the expressions. For example, let $\widetilde{R}^{(0)}$ and $\widetilde{R}^{(1)}$ denote the costs based on the filter norm; one can get the upper and lower bounds
\[
\frac{1}{n}\max_{x\in\Omega_-}\mathrm{Rank}_m(Jf(x))\leq \widetilde{R}^{(0)}(f;\Omega) \leq \frac{1}{n}\mathrm{Rank}_\text{CBN}(f;\Omega)
\]
and the regularity control becomes
\[
\widetilde{R}^{(1)}(f;\Omega) \geq \frac{2}{n}\sum_{s_{t,c}\neq 0}\tm_t^{-2}\log(s_{t,c}\tm_t).
\]

\section{Auxiliary Lemmas}
This section proves a version of Theorem 2.1 in \cite{AroraBMM16} for the CNN case and may be of independent interest.
\begin{lemma}[Bounded depth]
\label{lemma:bounded_depth}
    For any TEPL function $F:\Omega\subset\R^{n\times c_{in}}\rightarrow\R^{n\times c_{out}}$ and $L\geq \ceil{\log_2(nc_{in}+1)}+2$, there is a CNN $f_\theta=F$ with widths $\{c_\ell\}_{\ell=1}^L$ and parameter $\theta=\{w_\ell, b_\ell\}_{\ell=1}^L$.
\end{lemma}
\begin{proof}
    Since the input domain $\Omega$ is compact, without loss of generality, we may assume $x$ is positive. Consider $f:\R^{n\times c_{in}}\rightarrow\R^{c_{out}}$ with $f(x) = F_1(x)\equiv F(x)_{1,:}$ i.e. the first input at every channel. Note $f$ is piecewise linear and $F_p(x)=f(T_{-p}x)$ where $(T_{-p}x)_{i,:} = x_{i-p,:}$ is the translation of $x$ by $-p$. By Theorem 2.1 in \cite{AroraBMM16}, there is a ReLU fully-connected network $f^{FC}_{\boldsymbol{A}}=f$ with depth $L-1$, widths $\{n_\ell\}_{\ell=1}^{L-1}$, and parameter $\boldsymbol{A}=\{A_\ell, d_\ell\}_{\ell=1}^L$. Then we can construct an $L$-layer no-pooling ConvNet to represent $F$ as follows:

    \textbf{(First layer)} Construct $w_1\in\R^{n\times nc_{in}\times c_{in}}$ with $(w_1)_{p,c,s}=\delta_{p=n-\floor{c/c_{in}}}$ and $b_1=0$ such that 
    \[
    \alpha_1(x)_{p,:} = \talpha_1(x)_{p,:} = \mathrm{vec}(T_{-p}x),\quad p=0,\dots,n-1.
    \]
    where the first equality follows from assuming $x$ is positive. (i.e. This is done by shifting the identity convolution filter by $-p$ at channel $c\in\{pc_{in},\dots,(p+1)c_{in}-1\}$, then convolve with $m^{-1}$.) Then we treat each of the $T_{-p}x$ ``independently" in the following layers.

    \textbf{(Following layers)} For $\ell > 1$, construct $w_\ell\in\R^{n\times n_{\ell}\times n_{\ell-1}}$ by letting $(w_{\ell})_{i,c,s} = (A_\ell)_{c,s}\delta_{i=0}$ and $(b_\ell)_c = (d_\ell)_c$ for $i=0,\dots,n-1, c=1,\dots, n_\ell$, and $s=1,\dots, n_{\ell-1}$. One can verify that for $p=0,\dots,n-1$,
    \begin{align*}
        \talpha_\ell(x)_{p,c} &= (b_\ell)_c + \sum_{s=1}^{n_{\ell-1}}\sum_{i=0}^{n-1} (w_\ell)_{i,c,s}\alpha_{\ell-1}(x)_{p-i,s}\\
        &= (d_\ell)_c + \sum_{s=1}^{n_{\ell-1}}(A_\ell)_{c,s}\alpha_{\ell-1}(x)_{p,s}
    \end{align*}
    \[
    \implies \talpha_\ell(x)_{p,:} = A_\ell\alpha_{\ell-1}(x)_{p,:} + d_\ell.
    \]
    Now at the output layer $\ell=L$, by construction of $\boldsymbol{A}$, we have $\alpha_L(x)_{p,:} = f(\alpha_1(x)_{p,:}) = F_p(x)$. Hence the output of this CNN is $F(x)$.
\end{proof}

\end{document}